\documentclass[journal]{IEEEtran}
\usepackage[dvips]{graphicx}
\ifCLASSINFOpdf
\else
   \usepackage[dvips]{graphicx}
\fi

\usepackage[cmex10]{amsmath}

\usepackage{algorithm}
\usepackage{algorithmic} 
\usepackage{amssymb}
\usepackage{amsthm}

\usepackage{array}

\usepackage{mdwmath}
\usepackage{mdwtab}

  \usepackage[caption=false,font=footnotesize]{subfig}

\begin{document}
%
\title{Cascade Learning by Optimally Partitioning}

\author{Yanwei~Pang,~\IEEEmembership{Senior~Member,~IEEE,}
        Jiale~Cao,
        and~Xuelong~Li,~\IEEEmembership{Fellow,~IEEE,}
       
\thanks{Y. Pang and J. Cao are with the School of Electronic Information Engineering, Tianjin University, Tianjin 300072, China. e-mails: \{pyw,connor\}@tju.edu.cn}

\thanks{ X. Li is with the Center for OPTical IMagery Analysis and Learning (OPTIMAL), State Key Laboratory of Transient Optics and Photonics, Xi'an Institute of Optics and Precision Mechanics, Chinese Academy of Sciences, Xi'an 710119, Shaanxi, P. R. China. e-mail: xuelong\_li@opt.ac.cn.}
}

\markboth{}%
{Shell \MakeLowercase{\textit{et al.}}: Bare Demo of IEEEtran.cls for Journals}

\maketitle

\begin{abstract}

Cascaded AdaBoost classifier is a well-known efficient object detection algorithm. The cascade structure has many parameters to be determined. Most of existing cascade learning algorithms are designed by assigning detection rate and false positive rate to each stage either dynamically or statically. Their objective functions are not directly related to minimum computation cost. These algorithms are not guaranteed to have optimal solution in the sense of minimizing computation cost. On the assumption that a strong classifier is given, in this paper we propose an optimal cascade learning algorithm (we call it iCascade) which iteratively partitions the strong classifiers into two parts until predefined number of stages are generated. iCascade searches the optimal number $r_{i}$ of weak classifiers of each stage $i$ by directly minimizing the computation cost of the cascade. Theorems are provided to guarantee the existence of the unique optimal solution. Theorems are also given for the proposed efficient algorithm of searching optimal parameters $r_{i}$. Once a new stage is added, the parameter $r_{i}$ for each stage decreases gradually as iteration proceeds, which we call decreasing phenomenon. Moreover, with the goal of minimizing computation cost, we develop an effective algorithm for setting the optimal threshold of each stage classifier. In addition, we prove in theory why more new weak classifiers are required compared to the last stage. Experimental results on face detection demonstrate the effectiveness and efficiency of the proposed algorithm.
\end{abstract}


%
%
\IEEEpeerreviewmaketitle

\section{Introduction}
\IEEEPARstart{R}{OBUST} and real-time object detection  is a key problem in computer vision tasks such as vision-based Human Computer Interaction (HCI), video surveillance and biometrics. Robustness of an object detection system is mainly governed by the robustness of extracted features and the generalization ability of employed classifiers. The detection efficiency is determined by the types of features, the manner of the features to be extracted, and the structure of the classifiers \cite{Shao_FeatureLearning_TNNLS2014},  \cite{Bian_MinimizingNearest_TNNLS2014}, \cite{Yu_AUnifiedLearning_TTNLS2014}. For example, it is well known that the features can be computed by the trick of integral image, which is suitable for efficient object detection. However, the structure of classifiers is also important for efficient object detection. For example, AdaBoost classifiers with cascade structure have greatly contributed to real-time face detection \cite{Brubaker_Design_IJCV_2008}, \cite{Zhang_MultiInstance_NIPS_2007}, \cite{Jun_LTF_PAMI_2013}, \cite{Gualdi_ParticleWindows_PAMI_2011}, human detection \cite{Pang_CoHOG_SMC_2012}, \cite{Zhu_HOG_CVPR_2006}, \cite{Wang_HOG-LBP_ICCV_2009}, \cite{Benenson_100_Frames_ICCV_2012}, \cite{Dollar_PedestrianDetection_PAMI_2012}, \cite{Paisitkriangkrai_AScalableStageWise_TNNLS2014}, \cite{Paisitkriangkrai_RandomBoost_TNNLS2014}, etc. With cascade structure, a large fraction of sub-windows can be rejected at early stages with a small number of weak classifiers. Only the sub-windows of true positives and those similar to true positives can arrive at later stages. However, how to design an optimal cascade structure is an open problem which is the focus of this paper.

Cascade learning is the process of determining the parameters of a cascade in order to improve the efficiency of AdaBoost classifier. The cascade parameters mainly include the number of stages, the number of weak classifiers in each stage, and the threshold for each stage. However, most of existing cascade learning methods are not directly formulated as a constrained optimization problem. Though more efficient than the non-cascade one, they are not guaranteed to be the best in the sense of maximizing detection efficiency under acceptable constraints. Usually, there are many hand-crafted parameters which are chosen according to one's our intuition and experience. The performance of the cascade AdaBoost relies on one's insight into the cascade structure. As Saberian and Vasconcelos mentioned \cite{Saberian_Boosting_NIP_2010}, the design of a good cascade can take up several weeks. In addition, some useful intuitions are not justified in theory.

To overcome the above problems, we formulate cascade learning as a process of learning the parameters of a cascade by minimizing the computation cost with some certain constraints.

In summary, the contributions and characteristics of the paper are as follows.
\begin{enumerate}
\item We transform the strong classifier of regular AdaBoost into an optimal cascade classifier. That is, the result of regular AdaBoost is the input of our cascade learning algorithm. In the sense of detection rate and rejection rate, we use cascade AdaBoost to approach its non-cascade one (i.e., regular AdaBoost) with minimum computation cost.
\item The objective function of our method is just the computation cost of a cascade. In contrast, most of the existing algorithms are designed by empirically assigning detection rate and false positive rate to each stage either dynamically or statically. Existence and uniqueness of the optimal solution are analytically proved.
\item To design a one-stage cascade structure, we propose to partition the strong classifier $H(\mathbf{x})$, a combination of weak classifiers $h_{1},...,h_{T}$, into left part $H_{L}(\mathbf{x},r_{1})$ and right part $H_{R}(\mathbf{x},r_{1})$ at partition point $r_{1}$ (see Algorithm 1 and Fig. 1). The optimal partition point $r_{1}$ is found by minimizing the objective function $f_{1}(r)$ which stands for the computation cost of the cascade classifier. We theoretically (i.e., Theorem \ref{thm01}) prove  that $f_{1}(r)$ exists a unique solution. Moreover, we give a theorem (i.e., Theorem \ref{thm02}) that gives a rough estimation of the optimal solution.
\item To design a two-stage cascade structure, we propose to further partition right classifier $H_{R}(\mathbf{x},r_{1})$ into two parts at partition point $r_{2}$. The partition iteratively continues (see Fig. 4). This algorithm  is not globally optimal if $r_{1}$ is fixed while $r_{2}$ is considered as a variable. To obtain global optimization, we further jointly model the computation cost $f(r_{1},r_{2})$ with variables both $r_{1}$ and $r_{2}$. We prove that $f(r_{1},r_{2})$ has a unique minimum solution (see Theorem \ref{thm07}). An iterative optimization algorithm (i.e., Algorithm 2) is proposed to find the optimal solution. Theoretical analysis (i.e.,  Theorems \ref{thm09}-\ref{thm12}) is given that $r_{1}$ decreases in each iteration where $r_{2}$ is fixed and $r_{2}$ decreases in each iteration where $r_{1}$ is fixed. We call it decreasing phenomenon. Such globally optimal two-stage cascade learning algorithm can be easily generalized to multi-stage one (i.e., Algorithm 3). 
\item Moreover, we contribute to learning the optimal threshold $t_{i}$ of each stage classifier for minimizing computation cost $f_{S}$ of the cascaded classifier. We prove that the computation cost decreases with the stage threshold $t_{i}$ (i.e., Theorem \ref{thm13}). Based on this theorem, we develop an effective threshold learning algorithm (i.e., Algorithm 4) whose core is properly decreasing $t_{i}$. Though this algorithm is not globally optimal, it is very effective. We call the proposed algorithm (i.e., Algorithm 4 and the procedure in Fig. 10) iCascade. 
\item We prove in theory why more new weak classifiers are required compared to the previous stage (i.e., Theorem \ref{thm05}). In addition, we also theoretically prove why cascade AdaBoost is more efficient than its non-cascade one. Though the results and phenomena can be intuitively understood, we are the first to theoretically justify them to be the best of our knowledge. 
\end{enumerate}

\section{Related Work}
This section briefly reviews some existing work related to cascade leaning.

Most of existing cascade learning algorithms can be called DF-guided (where "DF" stands for Detection rate and False positive rate) method pioneered by Viola and Jones \cite{Viola_face_detection_IJCV_2004}. In the learning step, DF-guided method selects weak classifiers step by step until predefined minimum acceptable detection rate and maximum acceptable false positive rate are both satisfied. We call this method VJCascade \cite{Viola_face_detection_IJCV_2004} .

Variants of VJCascade have been proposed to select and organize weak classifiers. BoostChain \cite{Xiao_Boosting_Chain_CVPR_2003} improves VJCascade by reusing the ensemble score from previous stages to enhance current stage. Brubaker et al. \cite{Brubaker_Design_IJCV_2008} called such a technique BoostChain recycling. Similar to BoostChain, SoftCascade also allows for monotonic accumulation of information as the classifier is evaluated \cite{Bourdev_Soft_Cascade_CVPR_2005}. In Multi-exit AdaBoost \cite{Pham_Multi-exit_Asymmetric_CVPR_2005}, node classifier also shares overlapping sets of weak classifiers. FloatBoost \cite{Li_Floatboost_PAMI_2004} as well as Boost-Chain uses DF-guided strategy to design the cascade. But different from VJCascade, FloatBoost uses backtrack mechanism to eliminate the less useful or even detrimental weak classifiers. Wu et al. \cite{Wu_Asymmetric_Learning_PAMI_2008} employed Forward Feature Selection (FFS) algorithm to greedily select features. Wang et al. \cite{Wang_ATCB_NNLS_2012} developed an asymmetric learning algorithm for both feature selection and ensemble classifier learning. FisherBoost \cite{Shen_Effective_Node_IJCV_2013} uses column generation technique to implement totally-corrective boosting algorithm. To decrease the training burden caused by the large number of negative samples and over-complete features (e.g., Haar-like features), some algorithms use only a random subset of the feature pool \cite{Brubaker_Design_IJCV_2008}, \cite{Bourdev_Soft_Cascade_CVPR_2005}. 

Endeavor has also been devoted to adjust the thresholds of stages of a cascade structure which is also called the thresholds of node classifiers. On the assumption that a full cascade has been trained by VJCascade algorithm, Luo \cite{Luo_Optimization_Design_CVPR_2005} proposed to jointly optimize the setting of the thresholding parameters of all the node classifiers within the cascade. Waldboost algorithm utilizes an adaptation of Wald's sequential probability ratio test to set stage thresholds \cite{Sochman_Waldboost-learning_CVPR_2005}. Brubaker et al. proposed a linear program algorithm to select weak classifiers and threshold of a node classifier \cite{Brubaker_Design_IJCV_2008}, \cite{Brubaker_Towards_Optimal_CVPR_2006}.

Though most of existing methods are DF-guided, computation-cost guided (i.e., CC-guided) methods were also developed. Chen and Yuille \cite{Chen_Time-Efficient_CVPR_2005} gave a criterion for designing a time-efficient cascade that explicitly takes into account the time complexity of tests including the time for pre-processing. They designed a greedy algorithm to minimize the criterion. But each stage in this method is constrained to detect all positive examples, which leads it to miss opportunity to improve detection efficiency \cite{Chen_Classifier_Cascade_AIS_2012}. The loss function of Cronus cascade learning algorithm is a tradeoff between accuracy (training error) and computation cost \cite{Chen_Classifier_Cascade_AIS_2012}. CSTC (i.e., Cost-Sensitive Tree of Classifiers) combines regularized training error and computation cost into a loss function \cite{Xu_Cost-Sensitive_ICML_2013}. Compared to VJCascade-like method, CSTC is suitable for balanced classes and specialized features \cite{Xu_Cost-Sensitive_ICML_2013}.

In contrast to the above methods, the objective function (i.e., loss function) of our method is just the computation cost and the detection accuracy can be naturally guaranteed. In addition, global solution instead of local one can be obtained in our method.

\section{Proposed Method: One-Stage Cascade}
The goal of cascade learning is to lean a cascade structure in order to correctly reject negative sub-windows and accept positive sub-windows as fast as possible. Generally, the cascade structure is determined by the number of stages and the number of weak classifiers in each stage. 

Most of existing methods design or learn the cascade structure by assigning minimum acceptable detection rate and maximum acceptable false positive rate for each stage. In this paper, we propose a novel cascade learning method in which it is not necessary to assign such acceptable detection rates and false positive rates. Instead, we learn the parameters of a cascade by directly minimizing the computation cost.

In this section, we describe the proposed one-stage cascade learning algorithm which is the foundation of our multi-stage cascade learning algorithm.

\subsection{Testing Stage}
In our method, cascade AdaBoost is considered as an estimation of regular AdaBoost. A good cascade structure can achieve the same detection accuracy as AdaBoost with small computation cost. Therefore, we begin with describing the form of the strong classifier of regular AdaBoost.

Let $H(\mathbf{x})$ be the strong classifier obtained by an AdaBoost algorithm. The strong classifier $H(\mathbf{x})$ is composed of \textsl{T} weak classifier $h_i(\mathbf{x)}\in\{1,-1\}$ with weights $\alpha_i$:
\begin{equation}
\label{eq01}
H(\mathbf{x})=\sum\limits_{i=1}^{T}\alpha_{i}h_{i}(\mathbf{x}).
\end{equation}
Generally, the weights of the weak classifiers satisfy 
\begin{equation}
\label{eq02}
\alpha_{1}>\alpha_{2}>\cdots>\alpha_{T}>0,
\end{equation}
and
\begin{equation}
\label{eq03}
\sum\limits_{i=1}^{T}\alpha_{i}=1.
\end{equation}
Let $l(\mathbf{x})\in\{1,-1\}$ be the class label of a feature vector. The decision rule of the strong classifier $H(\mathbf{x})$ is:
\begin{equation}
\label{eq04}
l(\mathbf{x})=
\begin{cases}
1,&\text{if $H(\mathbf{x})=\sum\limits_{i=1}^{T}\alpha_{i}h_{i}(\mathbf{x})>t$}, \\
-1,&\text{if $H(\mathbf{x})=\sum\limits_{i=1}^{T}\alpha_{i}h_{i}(\mathbf{x})\leq t$},
\end{cases}
\end{equation}
where $t$ is a threshold balancing the detection rate and false positive rate. 

In one-stage cascade structure, there is only one stage in which a small number (i.e., $r$) of weak classifiers are combined for classification. The core of the proposed one-stage cascade is to determine an optimal $r$ which divides the strong classifier $H(\mathbf{x})$ into left part $H_{L}(\mathbf{x})$ and right part $H_{R}(\mathbf{x})$:
\begin{gather}
\label{eq05}
H(\mathbf{x})=H_{L}(\mathbf{x})+H_{R}(\mathbf{x}), \\
\label{eq06}
H_{L}(\mathbf{x})=\sum\limits_{i=1}^{r}\alpha_{i}h_{i}(\mathbf{x}), \\
\label{eq07}
H_{R}(\mathbf{x})=\sum\limits_{i=r+1}^{T}\alpha_{i}h_{i}(\mathbf{x}).
\end{gather}
To reject true negative sub-windows with less computation cost, we propose to use the maximum of $H_{R}(\mathbf{x})$ to approximate the value of $H_{R}(\mathbf{x})$:
\begin{equation}
\label{eq08}
\max H_{R}(\mathbf{x})=\max_{\mathbf{x}}\left(\sum\limits_{i=r+1}^{T}\alpha_{i}h(\mathbf{x})\right)=\sum\limits_{i=r+1}^{T}\alpha_{i}.
\end{equation}
We denote the maximum by $\max H_{R}(\mathbf{x})$. With $\max H_{R}(\mathbf{x})$, it is guaranteed that all the true negative sub-windows can be correctly rejected if the following inequality holds:
\begin{equation}
\label{eq09}
H_{L}(\mathbf{x},r)+\max H_{R}(\mathbf{x},r)\leq t.
\end{equation}
That is, some sub-windows can be rejected by using merely $H_{L}(\mathbf{x})$ and $\max H_{R}(\mathbf{x})$ instead of both $H_{L}(\mathbf{x})$ and $H_{R}(\mathbf{x})$. Consequently, the computation cost is significantly reduced.

The rest sub-windows not satisfying (\ref{eq09}) have to be classified using both $H_{L}(\mathbf{x})$ and $H_{R}(\mathbf{x})$ (i.e., the strong classifier). If the sum of $H_{L}(\mathbf{x})$ and $H_{R}(\mathbf{x})$ is not larger than $t$, i.e.,
\begin{equation}
\label{eq10}
H_{L}(\mathbf{x},r)+H_{R}(\mathbf{x},r)=H(\mathbf{x})\leq t,
\end{equation}
then the sub-window corresponding to the feature vector $\textbf{x}$ can be finally classified as negative sub-window. Otherwise (i.e., the sum is larger than $t$), it is classified as positive sub-window. The algorithm of one-stage cascade is given in Algorithm 1. Equivalently, the flow-chart is shown in Fig. 1. Note that $t-\max H_{R}(\mathbf{x},r)$ can be viewed as the threshold for $H_{L}(\mathbf{x},r)$. The issue of how to set the threshold is addressed in Section V.C.
\begin{figure}[!t]
\label{FigStageCascade}
\centering
\includegraphics{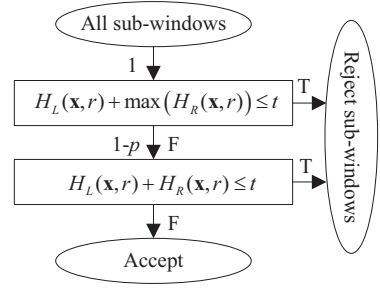}
\caption{Proposed method: one stage cascade AdaBoost with a given $r$.} 
\end{figure}
\begin{algorithm}[!t]
\caption{One-stage cascade}
\begin{algorithmic}[1]
\STATE if $H_{L}(\textbf{x},r)+\max H_{R}(\textbf{x},r)\leq t$, 
\STATE ~~~then $l(\textbf{x})=-1$,
\STATE else (i.e., $H_{L}(\textbf{x},r)+\max H_{R}(\textbf{x},r)>t$)
\STATE ~~~if $H_{L}(\textbf{x},r)+H_{R}(\textbf{x},r)\leq t$
\STATE ~~~~~~$l(\text{x})=-1$,
\STATE ~~~else (i.e., $H_{L}(\textbf{x},r)+H_{R}(\textbf{x},r)>t$)
\STATE ~~~~~~$l(\textbf{x})=1$.
\end{algorithmic}
\end{algorithm}

\subsection{Training Stage: How to Select an Optimal $r$}
In Fig. 1, it is assumed that $r$ and $\max H_R(\textbf{x},r)$ are given. In this sub-section, we describe how to choose an optimal $r$. $\max H_R(\textbf{x},r)$ can be easily computed from training samples once $r$ is given. In the training stage of cascade learning, it is assumed that the strong classifier $H(\textbf{x})=\sum_{i=1}^{T}\alpha_i h_i(\textbf{x})$ is obtained by a regular AdaBoost Algorithm. 

Given $r$, a $p$ fraction of true negative sub-windows can be rejected by using left classifier $H_{L}(\textbf{x})$ (i.e., (\ref{eq09})). The fraction $p$ is called rejection rate and defined by:
\begin{equation}
\label{eq11}
p(\mathbf{x})=\dfrac{\sum\limits_{\bf{x}}I(\sum\limits_{i=1}^{r}\alpha_{i}h_{i}(\mathbf{x})+\max H_{R}(\mathbf{x},r)\leq t)}{\sum\limits_{\mathbf{x}} I(l(\mathbf{x})==-1)},
\end{equation}
where $I(condition)$ is 1 if the condition is satisfied and 0 otherwise. $\sum_{\mathbf{x}} I(l(\mathbf{x})==-1)$ is the number of all true negative sub-windows. Eq. (\ref{eq11}) shows that $p$ is dependent on $r$.

Obviously, the fraction of true negative sub-windows classified by using both left and right classifiers is $1-p$. The criterion for choosing $r$ is to minimize the overall computation cost $f$ consisting of the cost $f^{L}$ of computing $H_{L}(\mathbf{x},r)$ in (\ref{eq09}) and the cost $f^{R}$ of computing  both $H_{L}(\mathbf{x},r)$ and $H_{R}(\mathbf{x},r)$ in (\ref{eq10}).

Suppose that all the weak classifiers have the same computation complexity. Then the computation cost is determined by the number of weak classifiers. A fact is that $f_{1}^{L}$ grows with $p$ and $r$:
\begin{equation}
\label{eq12}
f_{1}^{L}(r,p)=p(r+c),
\end{equation}
and $f_{1}^{R}$ grows with $1-p$ and $T$:
\begin{equation}
\label{eq13}
f_{1}^{R}(r,p)=(1-p)(T+2c).
\end{equation}

In (\ref{eq12}) and (\ref{eq13}), $c$ is the computational cost of checking either inequality (\ref{eq09}) or inequality (\ref{eq10}) holds. Usually, the computation cost C of a weak classifier is bigger than $c$. Let $C=1$, then $c<1$. Note that $c$ is not involved in computing $H_L(\mathbf{x},r)$ and $H_R(\mathbf{x},r)$. 

The goal is to minimize the following object function:
\begin{equation}
\label{eq14}
f_{1}(r,p)=f_{1}^{L}(r,p)+f_{1}^{R}(r,p)=p(r+c)+(1-p)(T+2c).
\end{equation}

To solve this optimization problem, it is necessary to reveal the relationship between $r$ and $p$. The parameters $r$ and $p$ are correlated and the correlation can be expressed as a function $p(r,\max H_{R}(\mathbf{x},r))$.

As $\max(H_{R}(\mathbf{x},r))$ (its upper bound is $\sum_{i=r+1}^{T}\alpha_{i}$) decreases with $r$, a larger number of negative sub-windows will be rejected by (\ref{eq09}). It is straightforward that the fraction $p$ of negative sub-windows satisfying (\ref{eq09}) grows with $r$.
Experimental results also show that $p$ monotonically increases with $r$. The relationship between $p$ and $r$ is nonlinear. Fig. 2 illustrates a typical trend that how $p$ varies with $r$. It can be seen that $p$ grows quickly from 0 to the value (e.g., 0.99) close to 1 when $r$ changes from 1 to a small value $r^*$ (e.g., 10). But $p$ becomes stable when $r$ is larger than $r^*$. The reason is that the first $r^*$ weak classifiers $h_{i}$ with larger weights $\alpha_{i}$ play much more important role than the rest weak classifiers.

Mathematically, $r^*$ is defined as the minimum $r$ which satisfies $p\approx1$ or equivalently $1-p(r)\leq \varepsilon$ with $\varepsilon$ being a small number (e.g., 0.01):
\begin{equation}
\label{eq15}
r^{*}=\arg\displaystyle\min_{r}\{r|1-p(r)\leq \varepsilon\}.
\end{equation}
We call $r^{*}$ the saturation point of $p(r)$.

Though $p(r)$ is in fact a high-order curve, it can be well modelled by combining of two linear functions: $p_{1}(r)=ar$ with $r<r^*$ and $p_{2}(r)=1$ with $r\geq r^*$(see Fig. 2). As $T$ is a large number, then $r^*\ll T$.
\begin{figure}[!t]
\label{FigRepresentativeForm}
\centering
\includegraphics{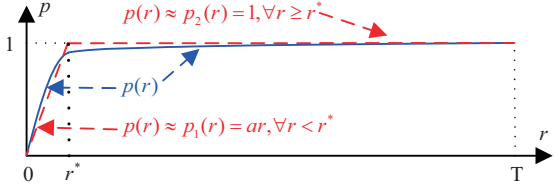}
\caption{A representative form of function $p(r)$ which can be simplified as combination of two linear functions: $p_{1}(r)=ar$ with $r<r^{*}$ and $p_{2}(r)=1$ with $r\geq r^{*}$.}
\end{figure}

It is reasonably assumed that the function $p(r)$ satisfies the fowllowing conditions:
\begin{gather}
\label{eq16}p(r_{1})<p(r_{2}),~if~r_{1}<r_{2},\\
\label{eq17}p'(r_{1})>p'(r_{2})\geq 0,~if~r_{1}<r_{2},\\
\label{eq18}p(T)=1,\\
\label{eq19}p(0)=0,\\
\label{eq20}p'(T)=0,\\
\label{eq21}p'(0)\gg 0,\\
\label{eq22}p'(1)\gg 0.
\end{gather}
(\ref{eq16}) states the monotonicity of $p(r)$.  (\ref{eq17}) tells that the slope of $p(r)$ decreases with $r$.  (\ref{eq20}) shows that the slope is zero at $r=T$ while it is extremely large at $r=1$. It is noted that (\ref{eq16})-(\ref{eq22}) will be used as assuption of the theorems of the proposed methods.

According to Fig. 2, $p(r)$ has the following properties: 
\begin{gather}
\label{eq23}r^*\ll T,~as~T~is~a~large~number,\\
\label{eq24}p(r)\approx p_{1}(r)=ar,~if~r<r^*,\\
\label{eq25}p(r)\approx p_{2}(r)=1,~if~r\geq r^*,
\end{gather}
which will be used as assumption of Theorem \ref{thm02}.

After each pair of $(r,p)$ are known, the value of $f_{1}$ can be obtained. Theorem \ref{thm01} tells that there exists a unique minimization solution.

\newtheorem{theorem}{Theorem}
\theoremstyle{plain}
\begin{theorem}
\label{thm01}
$f_{1}(r)=p(r)(r+c)+(1-p(r))(T+2c)$ has a unique minimum solution $r_{1}$. Moreover, $f_{1}(r)$ monotonically decreases with $r$ until $r=r_{1}$ and then increases with $r$.
\end{theorem}
\renewcommand{\proof}{\textbf{Proof.}}
\begin{proof}
We first prove the existence of the minimum solution and then give the evidence of the uniqueness of the minimum solution.

\textit{Existence:}

$\because f_{1}(r)=p(r)(r+c)+(1-p(r))(T+2c),$

$\therefore f'_{1}(r)=p'(r)(r-T-c)+p(r).$

Consider the value of the derivative $f'_{1}(r)$ when $r$ approaches 0:
\begin{equation}
\lim\limits_{r\to 0}f'_{1}(r)=f'_{1}(0)=p'(0)(0-c-T)+p(0).
\end{equation}
Because $p(0)=0$ (i.e., (\ref{eq19})) and $p'(0)\gg 0$ (i.e., (\ref{eq21}) ). Therefore, it holds:
\begin{equation}
\lim\limits_{r\to 0 }f'_{1}(r)=f'_{1}(0)=-p'(0)(T+c)<0.
\end{equation}

Now consider the value of the derivative $f'_{1}(r)$ when $r$ approaches $T$:
\begin{equation}
\begin{split}
\lim\limits_{r\to T}f'_{1}(r)=f'_{1}(T)=p(T)-p'(T)c\approx 1-0>0.
\end{split}
\end{equation}
Because $\lim\limits_{r\to 0}f'_{1}(r)<0$, $\lim\limits_{r\to T}f'_{1}(r)>0$ and $f'_{1}(r)$ is continuous function, it must exist a $r_{1}\in [1,T]$ such that $f'_{1}(r_{1})=0$. The $r_{1}$ is at least a local minimum, which shall be the global minimum if the local minimum is unique.

\textit{Uniqueness (Proof by contradiction):}

Suppose that there are two local minimums $r_{1}$ and $r_{2}$ with $r_{1}<r_{2}$. Then it holds that $f'_{1}(r_{1})-f'_{1}(r_{2})=0$.

Now investigate the value of $f'_{1}(r_{1})-f'_{1}(r_{2})=[p(r_{1})-p(r_{2})]-[p'(r_{1})(T+c-r_{1})-p'(r_{2})(T+c-r_{2})]$ if $r_{1}<r_{2}$ is true.\\
$\because r_{1}<r_{2},$\\
$\therefore p'(r_{1})>p'(r_{2})>0,~T+c-r_{1}>T+c-r_{2}>0,\\~~~~p(r_{1})<p(r_{2}),$\\
$\therefore f'(r_{1})-f'(r_{2})<0.$\\
This contradicts $f'(r_{1})-f'(r_{2})=0$, Therefore, $r_{1}<r_{2}$ is wrong. 

Similarly, we can prove that $r_{1}>r_{2}$ is wrong. Consequently, $r_{1}=r_{2}$ is true, meaning a unique solution.\qed
\end{proof}

Fig. 3  shows a representative form of $f_{1}(r)$, it has a unique minimum solution.
\begin{figure}[!t]
\label{FigOneStageComputation}
\centering
\includegraphics{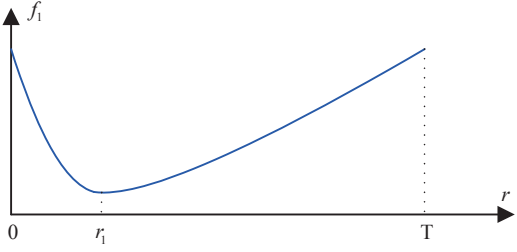}
\caption{A representative form of $f_{1}(r)$}
\end{figure}

\begin{theorem}
\label{thm02}
Let $r^*$ be a saturation point of $p(r)$ (see  (\ref{eq15})) and assume that $p(r)$ can be modelled by combining $p_{1}(r)=ar$ where $r<r^*$ with $p_{2}(r)=1$ where $r\geq r^*$(see Fig. 2
for illustration). Then the saturation point $r^*$ is the optimal minimum solution $r_{1}=\arg\min\limits_{r}f_{1}(r)$.
\end{theorem}

\begin{proof}
Note that $f_{1}^{L}(r)=p(r)(r+c)$,~$f_{1}^{R}(r)=(1-p(r))(T+2c)$,~and~$f_{1}(r)=f_{1}^{L}(r,p)+f_{1}^{R}(r,p)=p(r)(r+c)+(1-p(r))(T+2c).$

Case 1: For $r^*\leq r\leq T$, because $p_{2}(r)=1$ and $1-p_{2}(r)=0$, so we have $f_{1}^{L}(r)=p_{2}(r)(r+c)=r+c$,~$f_{1}^{R}(r)=0$, and hence $f_{1}(r)=r+c$. Therefore, the optimal solution $r_{R}^{*}$ for $r\geq r^{*}$ is $r^{*}$ itself. That is, 
\begin{equation}
\label{eq29}
r_{R}^{*}=\arg\min\limits_{r^*\leq r\leq T}f_{1}(r)=r^*.
\end{equation}

Case 2: For $0\leq r\leq r^*$, because $p(r)\approx p_{1}(r)=ar$, we have:
\begin{gather*}
f_{1}^{L}(r)=p_{1}(r)(r+c)=ar(r+c),\\
f_{1}^{R}(r)=(1-p(r))(T+2c)=(1-ar)(T+2c),\\
f_{1}(r)=f_{1}^{L}(r,p)+f_{1}^{R}(r,p)=ar^{2}-a(T+c)r+(T+2c),\\
f'_{1}=2ar-a(T+c)=0\Rightarrow\widetilde{r}=\arg\min\limits_{r}f_{1}(r)=(T+c)/2.
\end{gather*}
Because $r^{*}<\widetilde{r}$, $f_{1}(0)=T+2c$, and $f_{1}(r)$ monotonically decreases with $r$ when $r<\widetilde{r}$, the minimum value $r_{L}^{*}$ of $f_{1}(r)$ in the range of $0<r\leq r^{*}$ is $r^{*}$. That is,
\begin{equation}
\label{eq30}
r_{L}^{*}=\arg\min\limits_{0\leq r\leq r^{*}}f_{1}(r)=r^{*}.
\end{equation}

It is observed from (\ref{eq29}) and (\ref{eq30}) that the minimum solutions for $0\leq r\leq r^{*}$ and $r^{*}\leq r\leq T$ are identical to $r^{*}$. Consequently, $r^{*}=\arg\min\limits_{0\leq r\leq T}f_{1}(r)$.

Therefore, optimal minimum solution $r_{1}=\arg\min\limits_{r}f_{1}(r)$ is $r^{*}$.\qed
\end{proof}

\section{Proposed Method: Local-Minimum Based Multi-Stage Cascade}
In this section, we extend one-stage cascade learning to multi-stage cascade learning.

\subsection{Testing Stage}
From (\ref{eq05}), one can see that one-stage cascade is obtained by splitting $H(\mathbf{x})$ into $H_{L}(\mathbf{x},r_{1})$ and $H_{R}(\mathbf{x},r_{1})$ where $r_{1}$ is the optimal $r$ (i.e., $r_{1}=\arg\min\limits_r f_{1}(r)$). We add a superscript "1" to $H_{L}$ and $H_{R}$ so that one explicitly knows that $H_{L}^{1}(\mathbf{x},r_{1})$ and $H_{R}^{1}(\mathbf{x},r_{1})$ correspond to stage 1. Multi-stage cascade is obtained by iteratively splitting the right classifier $H_{R}^{i}$.

As shown in Fig. 4, the sub-windows not rejected by stage 1 are fed to stage 2. The second stage is obtained by further dividing the right classifier $H_{R}^{1}(\mathbf{x},r_{1})$ into two parts at the partition point $r_{2}$ ($r_{2}>r_{1}$):
\begin{gather}
\label{eq31}H_{R}^{1}(\mathbf{x},r_{1})=H_{L}^{2}(\mathbf{x},r_{2})+H_{R}^{2}(\mathbf{x},r_{2}),\\
\label{eq32}H_{L}^{2}(\mathbf{x},r_{2})=\sum\limits_{i=r_{1}+1}^{r_{2}}\alpha_{i}h_{i}(\mathbf{x}),\\
\label{eq33}H_{R}^{2}(\mathbf{x},r_{2})=\sum\limits_{i=r_{2}+1}^{T}\alpha_{i}h_{i}(\mathbf{x}).
\end{gather}

In stage 2, the sub-windows are rejected if the following inequality holds:
\begin{equation}
\label{eq34}
H_{L}^{1}(\mathbf{x},r_{1})+\left[H_{L}^{2}(\mathbf{x},r_{2})+\max\left(H_{R}^{2}(\mathbf{x},r_{2})\right)\right]\leq t.
\end{equation}
(\ref{eq34})  is equivalent to
\begin{equation}
\label{eq35}
H_{L}(\mathbf{x},r_{2})+\max\left(H_{R}(\mathbf{x},r_{2})\right)\leq t,
\end{equation}
because 
\begin{equation}
H_{L}(\mathbf{x},r_{2})=H_{L}^{1}(\mathbf{x},r_{1})+H_{L}^{2}(\mathbf{x},r_{2}).
\end{equation}
But (\ref{eq35}) is more time-consuming than (\ref{eq34}) because $r_{2}$ $(r_{2}>r_{1})$ weak classifiers are used to compute $H_{L}(\mathbf{x},r_{2})$ in (\ref{eq35}) whereas in (\ref{eq34}) $H_{L}^{1}(\mathbf{x},r_{1})$ has been computed in stage 1 and $H_{L}^{2}(\mathbf{x},r_{2})$ can be efficiently computed using as small as $(r_{2}-r_{1})$ weak classifiers where $H_{L}^{1}(\mathbf{x},r_{1})$ can be reused in stage 2.

Analogously, the left classifier in stage $i-1$ can be represented by the left and right classifiers in stage $i$:
\begin{equation}
H_{R}^{i-1}(\mathbf{x},r_{i-1})=H_{L}^{i}(\mathbf{x},r_{i})+H_{R}^{i}(\mathbf{x},r_{i}).
\end{equation}

\begin{figure}[!t]
\label{FigLocalMinimumProcess}
\centering
\includegraphics[width=3in]{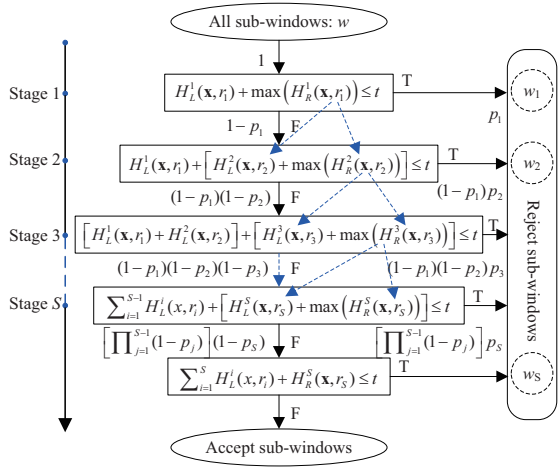}
\caption{Testing process of local-minimum based multi-stage cascade AdaBoost.}
\end{figure}

The block diagram of the multi-stage cascade is shown in Fig. 4 where the rejection rate $p_{i}$ is the ratio of sub-windows rejected in Stage $i$. In stage 1, $p_{1}$ fraction of sub-windows are directly rejected and $1-p_{1}$ fraction of sub-windows are fed to stage 2. Among the $(1-p_{1})w$ sub-windows, $p_{2}$ fractions are rejected by stage 2 and $1-p_{2}$ fractions are considered as positive-class candidates and therefore are fed to stage 3. This means that $(1-p_{1})(1-p_{2})$ fraction of total $w$ sub-windows are to be classified by stage 3. Because $p_{i}$ in stage $i$ is dependant on $p_{i-1}$ in stage $i-1$, we explicitly express $p_{i}$ as $p(r_{i}|r_{i-1})$ when necessary. Specifically, the rejection rate $p(r_{i}|r_{i-1})$ is defined as:
\begin{equation}
p(r_{i}|r_{i-1})\\=\dfrac{I(\sum\limits_{k=1}^{r_{i}}\alpha_{k}h_{k}(\mathbf{x})+\max(H_{R}(\mathbf{x},r_{i}))\leq t)}{I(\sum\limits_{k=1}^{r_{i-1}}\alpha_{k}h_{k}(\mathbf{x})+\max(H_{R}(\mathbf{x},r_{i-1}))>t)}.
\end{equation}

Fig. 5 shows two representative curves of $p(r_{i}|r_{i-1})$. The properties of $p(r_{i}|r_{i-1})$ are summarized as follows:
\begin{gather}
\label{eq39}p'(r_{i}|r_{i-1})\geq 0,\\
\label{eq40}\lim\limits_{r_{i}\to r_{i-1}}p(r_{i}|r_{i-1})=0,\\
\label{eq41}\lim\limits_{r_{i}\to T}p(r_{i}|r_{i-1})=1,\\
\label{eq42}p(r_{i}|\widetilde{r}_{i-1})>p(r_{i}|r_{i-1}),~if~\widetilde{r}_{i-1}<r_{i-1},\\
\label{eq43}p'(r_{i}|\widetilde{r}_{i-1})<p'(r_{i}|r_{i-1}),~if~\widetilde{r}_{i-1}<r_{i-1}.
\end{gather}
\begin{figure}[!t]
\label{FigTwoForm}
\centering
\includegraphics{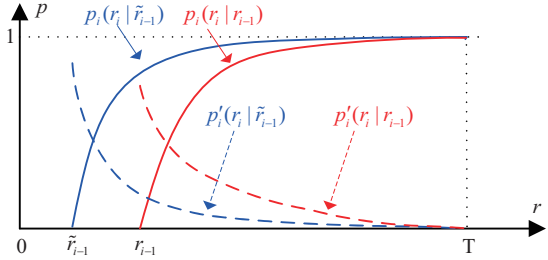}
\caption{The form of $p_{i}(r_{i}|r_{i-1})$ and its properties. If $\widetilde{r}_{i-1}<r_{i-1}$, then $p_{i}(r_{i}|\widetilde{r}_{i-1})>p_{i}(r_{i}|r_{i-1})$ and $p'_{i}(r_{i}|\widetilde{r}_{i-1})<p'_{i}(r_{i}|r_{i-1})$.}
\end{figure}

We give a theoretical guarantee (i.e., Theorem \ref{thm03}) that adding a stage results in reduction in computation cost if certain condition is satisfied.


\begin{theorem}
\label{thm03}
Let $r_{1},\ldots,r_{S}$ define an $S$ stage cascade structure whose computation cost is $f_{S}(r_{1},\ldots,r_{S})$:
\begin{equation}
\label{eq44}
f_{S}(r_{1},\ldots,r_{S})=\sum\limits_{i=1}^{S}f_{i}^{L}(r_{1},\ldots,r_{i})+f_{S}^{R}(r_{S}),
\end{equation}
where\\
\begin{equation}
\label{eq45}
f_{i}^{L}=\left[\prod\limits_{j=1}^{i-1}\left(1-p_{j}(r_{j}|r_{j-1})\right)\right]p_{i}(r_{i}|r_{i-1})(r_{i}+ic),
\end{equation}
\begin{equation}
\label{eq46}
f_{S}^{R}=\left[\prod\limits_{j=1}^{S}(1-p_{j}(r_{j}|r_{j-1}))\right](T+(S+1)c).
\end{equation}
Let $r_{1},\ldots,r_{S-1}$ define an $S-1$ stage cascade structure whose computation cost is $f_{S-1}(r_{1},\ldots,r_{S-1})$:
\begin{equation}
\label{eq47}
f_{S-1}(r_{1},\ldots,r_{S-1})=\sum\limits_{i=1}^{S-1}f_{i}^{L}(r_{1},\ldots,r_{i})+f_{S-1}^{R}(r_{S-1}).
\end{equation}
If $p_S(r_S|r_{S-1})>c/(T+c-r_S)$, then we have
\begin{equation}
f_{S-1}(r_{1},\ldots,r_{S-1})>f_{S}(r_{1},\ldots,r_{S-1},r_{S}).
\end{equation}
\end{theorem}
\begin{proof}\\
\begin{array}{l@{~}l}
\because & f_{S-1}(r_{1},\ldots,r_{S-1})-f_{S}(r_{1},\ldots,r_{S-1},r_{S})\\
&=f_{S-1}^{R}(r_{S-1})-f_{S}^{L}(r_{1},\ldots,r_{S})-f_{S}^{R}(r_{S})\\
&=\left[\prod\limits_{j=1}^{S-1}(1-p_{j})\right]\left[p_S (r_S|r_{S-1})(T+c-r_S)-c\right],
\end{array}
\begin{array}{l@{~}l}
\because & 1-p_{j}>0,~p_{S}>0,\\
\therefore & \text{if}~ p_S(r_S|r_{S-1})>c/(T+c-r_S), \text{then}~ f_{S-1}>f_{S}.
\end{array} 
\end{proof}\qed

Note that if the computational cost $c$ is omitted, then $f_{S-1} > f_S$ as long as $p_S(r_S|r_{S-1}) > 0$. In this case, it is optimal
that each stage contains a new weak classifier (i.e., the case
$S = T$ , $r1 = 1, r2 = 2, ..., r_T = T$ ). But $c \neq 0$ in practice, it
is necessary to let $S < T$ and find way to search the optimal
values of $r_1, ..., r_S$.

\subsection{Training Stage: How to Select Optimal $r_{i}$}
Section IV.A describes the testing stage of the proposed cascade method. Now we describe the training stage of the proposed method which is closely related to Section III.B.

\subsubsection{Existence and Uniqueness}
Investigating Fig. 4, one can find that the cascade structure is completely determined once $r_{1},\ldots,r_{S}$ are known. Therefore, the main task of the training stage is to find the optimal $r_{1},\ldots,r_{S}$.

The $r_{1}$ in stage 1 is obtained by the method in Section III.B. Given $r_{1}$, we learn the  best $r_{2}$:
\begin{equation}
r_{2}=\arg\min\limits_{r}f_{2}(r_{1},r)=\arg\min\limits_{r}f_{2}(r|r_{1}).
\end{equation}
Similar to the proof of Theorem \ref{thm01}, it can be  proved that $f_{2}(r|r_{1})$ has a unique solution.

Generally, $r_{i}$ is computed based on $r_{1},\ldots,r_{i-1}$:
\begin{equation}
\label{eq50}
\begin{split}
r_{i}&=\arg\min\limits_{r_{i-1}<r<T}f_{i}(r_{1},\ldots,r_{i-1},r)\\
&=\arg\min\limits_{r_{i-1}<r<T}f_{i}(r|r_{1},\ldots,r_{i-1}).
\end{split}
\end{equation}

In (\ref{eq50}), $f_{i}(r|r_{1},\ldots,r_{i-1})$ is used to describe the assumption that $r_{1},\ldots,r_{i-1}$ in the first $i-1$ stages are given. If $r_{1},\ldots,r_{i-1}$ and $p(r_{1}),p(r_{2}|r_{1}), \ldots, p(r_{i-1}|r_{i-2})$ are known, then $f_{i}(r|r_{1},\ldots,r_{i-1})$ will be in the similar form as $f_{1}(r)$ (see (\ref{eq14})):
\begin{equation}
\label{eq51}
\begin{array}{l@{~}l}
& f_{i}(r_{i}|r_{1},\ldots,r_{i-1})\\
&=\sum\limits_{j=1}^{i}f_{j}^{L}(r_{1},\ldots,r_{j})+f_{i}^{R}(r_{i})\\
&=\sum\limits_{j=1}^{i-1}f_{j}^{L}(r_{1},\ldots,r_{j})+f_{i}^{L}(r_{1},\ldots,r_{i})+f_{i}^{R}(r_{i})\\
&=\left[\sum\limits_{j=1}^{i-1}f_{j}^{L}(r_{1},\ldots,r_{j})\right]\\
&+\left[\prod\limits_{j=1}^{i-1}(1-p_{j}(r_{j}|r_{j-1}))\right]p_{i}(r_{i}|r_{i-1})(r_{i}+ic)\\
&+\left[\prod\limits_{j=1}^{i-1}(1-p_{j}(r_{j}|r_{j-1}))\right](1-p_{i}(r_{i}|r_{i-1}))(T+(i+1)c).
\end{array}
\end{equation}

Because the items in bracket in (\ref{eq51}) are constant, so $f_{i}(r|r_{1},\ldots,r_{i-1})$ is in the similar form as $f_{1}(r_{1})$. Therefore, as a corollary of Theorem \ref{thm01}, we have the following theorem:

\begin{theorem}
\label{thm04}
$\min\limits_{r}f_{i}(r_{1},\ldots,r_{i-1},r)=\min\limits_{r}f_{i}(r|r_{1},\ldots,\\r_{i-1})$ has a unique minimum solution $r_{i}\in [r_{i-1},T]$. Moreover, $f_{i}(r|r_{1},\ldots,r_{i-1})$ monotonically decreases with $r$ until $r=r_{i}$ and then increases with $r$.
\end{theorem}

Theorem \ref{thm04} implies that $f_{i}(r|r_{1},\ldots,r_{i-1})$ has the similar form as the curve in Fig. 3.

\subsubsection{Efficient Search}
The search range of $r_{i}$ is $(r_{i-1},T)$. However, because $f_{i}(r|r_{1},\ldots,r_{i-1})$ monotonically decreases with $r$ until $r=r_{i}$ and then increases with $r$, to find the unique minimum solution one can increase $r$ from $r_{i-1}$ with a small step and stop at the value once $f_{i}(r|r_{1},\ldots,r_{i-1})$ no longer decreases. Therefore, the practical range is less than $(r_{i-1},T)$.

The search range can be further reduced according to the following increasing phenomenon.

\begin{theorem}
\label{thm05}
If $r_{i}=\arg\lim\limits_{r_{i-1}<r<T}f_{i}(r|r_{1},\ldots,r_{i-1})$, $r_{i-1}=\arg\min\limits_{r_{i-2}<r<T}f_{i-1}(r|r_{1},\ldots,r_{i-2})$, $r_{i+1}=\arg\min\limits_{r_{i}<r<T}f_{i+1}(r|r_{1},\ldots,r_{i})$ and $2r_{i}-r_{i-1}<T$, then it holds:
\begin{gather}
r_{i}-r_{i-1}\leq r_{i+1}-r_{i},\\
r_{i+1}\geq 2r_{i}-r_{i-1}.
\end{gather}
We define $r_{0}=0$, so we have:
\begin{equation}
\label{eq54}
r_{2}\geq 2r_{1}.
\end{equation}
\end{theorem}
\begin{proof}
This theorem can be proved by using Theorem \ref{thm02} and the properties of $p_{i}(r|r_{i-1})$ and $p_{i+1}(r|r_{i})$

The curves of $p_{i+1}(r|r_{i})$ with $r>r_{i}$ and $p_{i}(r|r_{i-1})$ with $r>r_{i-1}$ have the similar shapes according to Fig. 2.  The difference between $p_{i}(r|r_{i-1})$ and $p_{i+1}(r|r_{i})$ can be characterized by their saturation points $r_{i}^{*}=\arg\min\limits_{r}\{r|1-p(r|r_{i-1})\leq\varepsilon\}$ and $r_{i+1}^{*}=\arg\min\limits_{r}\{r|1-p(r|r_{i})\leq\varepsilon\}$. Define increasing step $\Delta r_{i}^{*}=r_{i}^{*}-r_{i-1}$ and $\Delta r_{i+1}^{*}=r_{i+1}^{*}-r_{i}$, then $r_{i}^{*}=r_{i-1}+\Delta r_{i}^{*}$ and $r_{i+1}^{*}=r_{i}+\Delta r_{i+1}^{*}$ can be rewritten as:
\begin{gather}
\Delta r_{i}^{*}=\arg\min\limits_{\Delta r}\{\Delta r|1-p(r_{i-1}+\Delta r|r_{i-1})\leq\varepsilon\},\\
\Delta r_{i+1}^{*}=\arg\min\limits_{\Delta r}\{\Delta r|1-p(r_{i}+\Delta r|r_{i})\leq\varepsilon\}.
\end{gather}
Now investigate the curves of $p(r_{i-1}+\Delta r|r_{i-1})$ and $p(r_{i}+\Delta r|r_{i})$ (see Fig. 6). According to the property (i.e., (\ref{eq42})) of $p(r_{i}|r_{i-1})$, $\Delta r_{i+1}^{*}>\Delta r_{i}^{*}$ holds because $r_{i}>r_{i-1}$.
\begin{figure}[!t]
\label{FigDeta}
\centering
\includegraphics{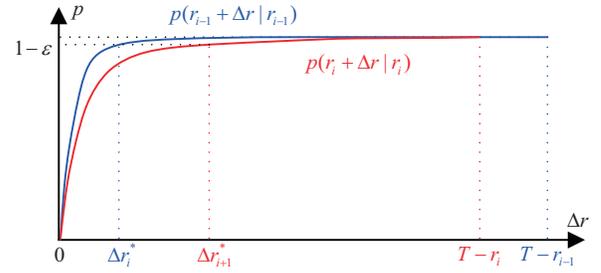}
\caption{$p(r_{i-1}+\Delta r|r_{i-1})$ and $p(r_{i}+\Delta r|r_{i})$.}
\end{figure}\qed
\end{proof} 

Fig. 7 illustrates the nature of Theorem \ref{thm05} where the objective functions and estimated optimal solution at saturation points are shown. The relationship of $r_{i-1}$, $r_{i}$, and $r_{i+1}$ are $r_{i-1}-r_{i-2}<r_{i}-r_{i-1}<r_{i+1}-r_{i}$ (i.e., $\Delta r_{i-1}<\Delta r_{i}<\Delta r_{i+1}$).
\begin{figure}[!t]
\label{FigIllustration}
\centering
\includegraphics{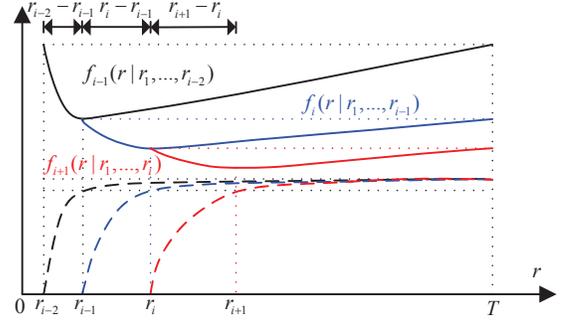}
\caption{Illustration of Theorem \ref{thm05}.}
\end{figure}

According to Theorem \ref{thm05}, if $r_{i}=\arg\min\limits_{r_{i-1}<r<T}f_{i}(r|r_{1},\ldots,r_{i-1})$ and $r_{i-1}=\arg\min\limits_{r_{i-2}<r<T}f_{i}(r|r_{1},\ldots,r_{i-2})$ are already known, then the search range for $r_{i+1}$ will be reduced to $[r_{i}+(r_{i}-r_{i-1}),T]$ (i.e., $[2r_{i}-r_{i-1},T]$) where $r_{i}-r_{i-1}$ is called increasing step. 

The training process is shown in Fig. 8 where the increasing phenomenon is used for efficient minimization.
\begin{figure}[!t]
\label{FigLocalCascadeLearning}
\centering
\includegraphics{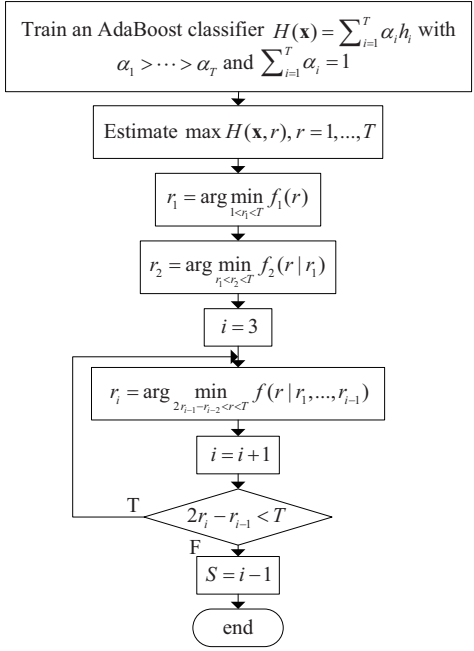}
\caption{Local-minimum based multi-stage cascade learning.}
\end{figure}

\section{Proposed Method: Joint-Minimum Based Multi-Stage Cascade}

\subsection{Existence and Uniqueness of a Jointly Optimal Solution}
The method in Section IV is a greedy optimization algorithm because it seeks 
an optimal $r_i $ on the condition that ($r_1 $,\ldots ,$r_{i - 1} )$ are 
known and fixed. The objective function is $f_i (r\vert r_1 ,...,r_{i - 1} 
)$, $i = 1,...,S$. In this section, we give an algorithm for jointly seeking 
the optimal ($r_1 $,\ldots ,$r_S )$ that globally minimizes the objective 
function $f(r_1 ,...,r_S )$ instead of $f_S (r\vert r_1 ,...,r_{S - 1} )$. 
That is, the goal of joint optimization is to find $(r_1^ * ,...,r_S^ * ) = 
\arg \mathop {\min }\limits_{r_1 ,...,r_S }f_S (r_1 ,...,r_S )$

For the sake of clarity, we start with establishing a globally optimal 
two-stage cascade structure. The globally optimal cascade structure with 
more than two stages will be extended from the two-stage one.

The goal of jointly optimal two-stage cascade learning aims at finding 
$(r_1^ * ,r_2^ * ) = \arg \mathop {\min }\limits_{r_1 ,r_2 }f_2 (r_1 ,r_2 
)$.

Obviously, if both $f_2 (r\vert r_1 ) = f_2 (r_1 ,r)$ and $f_2 (r\vert r_2 ) 
= f_2 (r,r_2 )$ have unique minimization solutions, then $f_2 (r_1 ,r_2 )$ has unique 
minimization solutions. $f_2 (r\vert r_1 )$ means the objective function of a two-stage cascade 
where the parameter $r_1 $ of stage 1 is known and the parameter $r$ of 
stage 2 is a unknown variable. $f_2 (r\vert r_2 )$ stands for the situation 
where the parameter $r_2 $ of stage 2 is known and the parameter $r$ of 
stage 1 is a variable. The theorems related to the jointly optimization are 
as follows.

\begin{theorem}
\label{thm06}
$\mathop {\min }\limits_r f_2 (r,r_2 ) = \mathop {\min 
}\limits_r f_2 (r\vert r_2 )$ has a unique minimum solution $r_1 $.
\end{theorem}
\begin{proof}
We first prove the existence of the minimum solution and 
then give the evidence of the uniqueness of the minimum solution.

\textit{Existence: }

\begin{array}{l@{~}l}
\because f_2 (r\vert r_2 )= & p_1 (r)(r+c) + (1 - p_1 (r))p_2 (r_2 \vert r)(r_2+2c)\\
& + (1 - p_1 (r))(1 - p_2 (r_2 \vert r))(T+3c),
\end{array}

\begin{array}{l@{~}l}
\therefore f'_{2}(r\vert r_{2}) = & p'_1 (r)(r-2c-T)+ p_1 (r)\\
&  + (T +c- r_2 ) p'_1 (r)p_2 (r_2 \vert r)\\
& - (T +c- r_2 )(1 - p_1 (r))p'_2 (r_2 \vert r).
\end{array}

Because the sum of rejected negative sub-windows of stage 1 and stage 2 is a 
const $\eta > 0$ once $r_2 $ is fixed:
\begin{equation}
\label{eq57}
p_1 (r) + (1 - p_1 (r))p_2 (r_2 \vert r) = p_1 (r_2 ) = \eta .
\end{equation}
Computing the derivative of $r$ to both sides of (\ref{eq57}) yields:
\begin{equation}
\label{eq58}
{p}'_1 (r)p_2 (r_2 \vert r) - (1 - p_1 (r)){p}'_2 (r_2 \vert r) = {p}'_1 
(r).
\end{equation}
Therefore, we can get $f'_{2}(r\vert r_{2})$ as
\begin{equation}
\label{eq59}
\begin{split}
{f}'_2 (r\vert r_2 ) = p_1 ^\prime (r)(r - r_2-c ) + p_1 (r).
\end{split}
\end{equation}

$\because\mathop {\lim }\limits_{r \to 0} p_1 (r) = 0$ (see (\ref{eq19})) and ${p}'_1 (r) \ge 
0$ (see (\ref{eq17})),

$\therefore\mathop {\lim }\limits_{r \to 0} {f}'_2 (r\vert r_2 ) = - (r_2+c) p_1 ^\prime 
(0) < 0$.

$\because\mathop {\lim }\limits_{r \to r_2+c } {f}'_2 (r\vert r_2 ) = p_1 (r_2+c ) > 0$.

$\because\mathop {\lim }\limits_{r \to r_2+c } {f}'_2 (r\vert r_2 ) > 0$, $\mathop 
{\lim }\limits_{r \to 0} {f}'_2 (r\vert r_2 ) < 0$ and ${f}'_2 (r\vert r_2 
)$ is a continuous function,

$\therefore$ It must exist a $r_1 \in [1, r_2+c )$ satisfying ${f}'_2 (r_1 \vert r_2 ) = 
0$ and $r_1 = \arg \min f_2 (r\vert r_2 )$.

\textit{Uniqueness (Proof by contradiction)}: 

Suppose there are two local minimum solutions $\tilde {r}_1 $ and $r_1 $ with $\tilde 
{r}_1 < r_1 $. Then it holds that ${f}'_2 (\tilde {r}_1 \vert r_2 ) - {f}'_2 
(r_1 \vert r_2 ) = 0$. 

Now investigate the value of ${f}'_2 (\tilde {r}_1 \vert r_2 ) - {f}'_2 (r_1 
\vert r_2 )$ = $\left[ {p_1 ^\prime (\tilde {r}_1 )(\tilde {r}_1 - r_2 -c) - p_1 ^\prime 
(r_1 )(r_1 - r_2 -c)} \right] + \left[ {p_1 (\tilde {r}_1 ) - p_1 (r_1 )} \right]$ if $\tilde {r}_1 < r_1 $ is true.

$\because\tilde {r}_1 < r_1,$

$\therefore{p}'(\tilde {r}_1 ) > {p}'(r_1 ) > 0, \tilde {r}_1 - r_2-c  < r_1 - r_2-c < 0,
0 < p(\tilde {r}_1 ) < p(r_1 ),$

$\therefore{f}'_2 (\tilde {r}_1 \vert r_2 ) - {f}'_2 (r_1 \vert r_2 ) < 0.$

This contradicts ${f}'_2 (\tilde {r}_1 \vert r_2 ) - {f}'_2 (r_1 \vert r_2 ) 
= 0$. Therefore, $\tilde {r}_1 < r_1 $ is wrong.

Similarly, we can prove that $\tilde {r}_1 > r_1 $ is wrong. Consequently, 
$\tilde {r}_1 = r_1 $ is true which means a unique solution in $r_1 \in [1, r_2+c )$. Because $c$ is smaller than 1 (see the statement below (\ref{eq13})), It is equivalent that the unique solution $r_1$ is in the range $[1, r_2 )$ \qed

\end{proof}

Theorem \ref{thm06} tells that if the information of stage 2 is given, then one can 
find an optimal parameter $r$ for stage 1 so that the computation cost $f_2 
$ of the final two-stage cascade is minimized. 

\begin{theorem}
\label{thm07}
$f_2 (r_1 ,r_2 )$ has a unique minimum solution $(r_1^ * 
,r_2^ * )$. 
\end{theorem}
\begin{proof}
Because both $\mathop {\min }\limits_r f_2 (r\vert r_1 ) = \mathop 
{\min }\limits_r f_2 (r_1 ,r)$ (see Theorem \ref{thm04}) and $\mathop {\min }\limits_r 
f_2 (r\vert r_2 ) = f_2 (r,r_2 )$ (see Theorem \ref{thm06}) have unique minimum 
solutions, so $\min f_2 (r_1 ,r_2 )$ has a unique minimum solution. That is $f_2 (r_1^ * ,r_2^ * ) = \mathop {\min }\limits_{r_1 ,r_2 } f_2 (r_1 ,r_2 ) = \mathop {\min }\limits_{r_1 } \mathop {\min }\limits_{r_2 } f_2 (r_2 \vert r_1 ) = \mathop {\min }\limits_{r_2 } \mathop {\min }\limits_{r_1 } f_2 (r_1 \vert r_2 )$ \qed
\end{proof}

It is straightforward to generalize Theorem \ref{thm07} to the following Theorem:

\begin{theorem}
\label{thm08}
$f_i (r_1 ,...,r_i )$ has a unique minimum 
solution $(r_1^ * ,...,r_i^ * )$.
\end{theorem}

\subsection{How to Search the Jointly Optimal Solutions}

\subsubsection{Algorithm}
Theorem \ref{thm08} guarantees the existence and uniqueness of jointly 
optimizing the stages of a cascade. In this section, we give algorithms 
(i.e., Algorithms \ref{thm02} and \ref{thm03}) for searching the solution and then theoretically justify 
the algorithms in theory. We start with the algorithm for optimizing a 
two-stage cascade and then generalize it to multi-stage one. 

The task of jointly optimizing a two-stage cascade can be expressed as $(r_1^ 
* ,r_2^ * ) = \arg \mathop {\min }\limits_{r_1 ,r_2 }f_2 (r_1 ,r_2 )$. 
The idea of our optimization method is shown in Algorithm 2.

\begin{algorithm}[!t]
\renewcommand{\algorithmicrequire}{\textbf{Input:}}
\renewcommand\algorithmicensure {\textbf{Output:} }
\renewcommand\algorithmicrepeat {\textbf{Iteration} }
\caption{ Globally optimal two-stage cascade learning.}
\begin{algorithmic}[1]
\REQUIRE ~~\\
Strong classifier $H({\rm {\bf x}}) = \sum\nolimits_{i = 1}^T {\alpha _i h_i ({\rm {\bf x}})} $ and its threshold $t$;\\
A set of true negative sub-windows$\{{\rm {\bf x}}\vert l({\rm {\bf x}}) = - 1\}$;
\ENSURE ~~\\ 
$(r_1^ * ,r_2^ * ) = \arg \mathop {\min }\limits_{r_1 ,r_2 } = f_2 (r_1 ,r_2 )$;
\STATE \textbf{Initialization}

\STATE Search the optimal solution $r_1 $ of $f_1 (r)$ for stage 1 in the range of $(1,T)$: $r_1 = \arg \mathop {\min }\limits_{1 < r < T} f_1 (r)$.

\STATE Given $r_1 $, search the optimal solution $r_2 $ of $f_2 (r\vert r_1 )$ for stage 2 in the range of $[2r_1 ,T)$: $r_2 = \arg \mathop {\min }\limits_{2r_1 \leq r < T} f_2 (r\vert r_1 )$. See Theorem \ref{thm05} for the reason of $r \geq 2r_1 $.

\REPEAT
\STATE Given $r_2 $, search the optimal solution $\tilde {r}_1 $ of $f_2 (r\vert r_2 )$ in the range of $1<r \leq r_1 $ for stage 1: $\tilde {r}_1 = \arg \mathop {\min }\limits_{1 < r \leq r_1 } f_2 (r\vert r_2 )$. Note that $\tilde {r}_1 < r_1 $ (see Theorem \ref{thm09}). An efficient search strategy is decreasing $r$ from $r_1 $ step by step until $f_2 (r\vert r_2 )$ does not decrease. $f \leftarrow f_2 (r_1 \vert r_2 )$.

\STATE Given $\tilde {r}_1 $, search the optimal solution $\tilde {r}_2 $ of $f_2 (r\vert \tilde {r}_1 )$ for stage 2 in the range of $\tilde {r}_1<r  \leq r_2 $: $\tilde {r}_2 = \arg \mathop {\min }\limits_{\tilde {r}_1 < r \leq r_2 } f_2 (r\vert \tilde {r}_1 )$. Note that $\tilde {r}_2 \leq r_2 $ (see Theorem \ref{thm12}). An efficient search strategy is decreasing $r$ from $r_2 $ step by step until $f_2 (r\vert \tilde {r}_1 )$ does not decrease. $\tilde {f} \leftarrow f_2 (r\vert \tilde {r}_1 )$.

\STATE Update $r_1 \leftarrow \tilde {r}_1 $, $r_2 \leftarrow \tilde {r}_2 $.
\UNTIL{$f - \tilde {f} \ge \mu$}

\RETURN $r_1^\ast \leftarrow \tilde {r}_1 $, $r_2^\ast \leftarrow \tilde {r}_2 $.
\end{algorithmic}
\end{algorithm}


The proposed Algorithm 2 is an alternative optimization procedure. In the 
initialization step, the solution $r_1 $ of the one-stage cascade learning is searched in the largest range $1 < 
r < T$: $r_1 = \arg \mathop {\min }\limits_{1 < r < T} f_1 (r)$. The value 
of $r_1 $ is shown in Fig. 9, where "{\#}1'' means that $r_1 $ is obtained 
firstly. The obtained $r_1 $ is used as the upper bound of the searching 
range for the better solution $\tilde {r}_1 $ in line 5 of Algorithm 2. After 
$r_1 $ is given, line 3 of Algorithm 2 searches the optimal solution $r_2 
$ of $f_2 (r\vert r_1 )$ for stage 2 in the range of $2r_1 \leq r < T$: $r_2 = 
\arg \mathop {\min }\limits_{2r_1 \leq r < T } f_2 (r\vert r_1 )$. Based on (\ref{eq54}), the search 
range starts from $2r_1 $. The value of $r_2 $ is shown in Fig. 9, where "{\#}2'' means that $r_2$ is the second value obtained by Algorithm 2.

In line 5 of Algorithm 2, $r_2 $ is given and the task is to search the 
optimal solution $\tilde {r}_1 $ of $f_2 (r\vert r_2 )$ in the range of $1 < 
r \leq r_1 $ for stage 1: $\tilde {r}_1 = \arg \mathop {\min }\limits_{1 < r \leq
r_1 } f_2 (r\vert r_2 )$. Because $r_1 \ll T$, the search range $1 < r \leq r_1 
$ is much smaller than the one (i.e., $1 < r < T)$ in line 2. Theorem \ref{thm09} guarantees $\tilde {r}_1 \leq r_1$ for the first round of iteration. $\tilde {r}_1 $ 
is the third value obtained by Algorithm 2 which is shown near "{\#}3'' 
in Fig. 9. Experimental results and intuitive analysis show that the 
absolute distance $\vert \tilde {r}_1 - r_1 \vert $ from $\tilde {r}_1 $ to 
$r_1 $ is much smaller than the absolute distance $\vert 1 - \tilde {r}_1\vert $ from  1 to $\tilde {r}_1 $, the search strategy of decreasing 
$r$ from $r_1 $ step by step until $f_2 (r\vert r_2 )$ does not 
decrease is more efficient than the one of increasing $r$ from 1 step by 
step until $f_2 (r\vert r_2 )$ does not increase.

In line 6 of Algorithm 2, $\tilde {r}_1 $ is given and the task is to search 
the optimal solution $\tilde {r}_2 $ of $f_2 (r\vert \tilde {r}_1 )$ in the range of $\tilde {r}_1 < r \leq r_2$ for 
stage 2: $\tilde {r}_2 = \arg \mathop {\min }\limits_{\tilde 
{r}_1 < r \leq r_2 } f_2 (r\vert \tilde {r}_1 )$. Because $r_2 < T$, the upper 
bound of the search range is much smaller than the one (i.e., $T)$ in line 3. Moreover, as iteration runs, the updated $r_2 $ 
becomes smaller and so the upper bound of search range for $\tilde {r}_2 $ 
becomes smaller too. Theorem \ref{thm10} guarantees $\tilde {r}_2 \leq r_2 $. The value 
of $\tilde {r}_2 $ is shown in Fig. 9 which is "{\#}4" obtained by Algorithm 2. Experimental results and intuitive analysis show that the 
absolute distance $\vert \tilde {r}_2 - r_2 \vert $ from $\tilde {r}_2 $ to 
$r_2 $ is much smaller than the absolute distance $\vert \tilde {r}_1 - \tilde {r}_2  \vert$ from $\tilde {r}_1$ to $\tilde {r}_2 $, the search strategy of decreasing 
$r$ from $r_2 $ step by step until $f_2 (r\vert \tilde {r}_1 )$ does not 
decrease is more efficient than the one of increasing $r$ from $\tilde {r}_1 
$ step by step until $f_2 (r\vert \tilde {r}_1 )$ does not increase. 

In the second round of iteration, because $\tilde {r}_2 \leq r_2 $, the 
parameter value $\tilde {\tilde {r}}_1 $ for stage 1 is obtained and shown 
in Fig. 9 with a label "{\#}5". According to Theorem \ref{thm11}, it is true that 
$\tilde {\tilde {r}}_1 \leq \tilde {r}_1 $. Subsequently, the parameter value 
$\tilde {\tilde {r}}_2 $ for stage 2 is obtained and shown in Fig. 10 with a 
label ``{\#}6''. According to Theorem \ref{thm10}, it is true that $\tilde {\tilde 
{r}}_2 \leq \tilde {r}_2 $.

The iteration stops if the difference between the value $f$ of objective 
function in line 5 of Algorithm 2 and the one $\tilde {f}$ in in line 6 of Algorithm 2 is equal to or smaller than the threshold $\mu \ge 0$.

\begin{figure}[!t]
\label{FigIntermediate}
\centering
\includegraphics{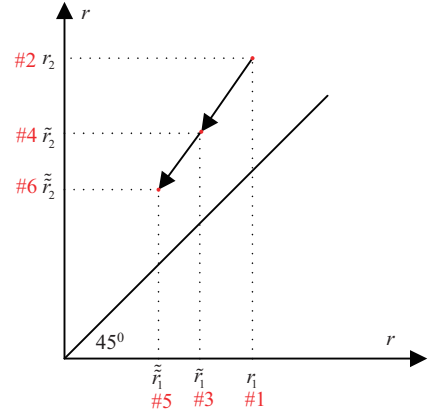}
\caption{Illustration of the intermediate values obtained by Algorithm 2. $r_1 $ and $r_2 $ are the outputs of Initialization. The sequence of the stage parameters are updated in the following turn: $r_1  \to r_2  \to \tilde {r}_1  \to \tilde {r}_2  \to \tilde {\tilde {r}}_1  \to \tilde {\tilde {r}}_2 $ with $\tilde {\tilde {r}}_1 < \tilde {r}_1 < r_1 $ and $\tilde {\tilde {r}}_2 < \tilde {r}_2 < {r}_2$}
\end{figure}

\textit{Decreasing Phenomenon:} Fig. 9 shows an interesting phenomenon: (1) 
Once a new stage 2 is added, the parameter $r_1 $ of stage 1 should be 
updated by decreasing $r_1 $ to a smaller number $\tilde {r}_1 $ so that the 
computation cost is minimized. (2) Once the number of stages is fixed, the 
parameter for each stage decreases gradually as iteration proceeds.

\subsubsection{Justification of the Algorithm}
Theorems \ref{thm09}-\ref{thm12} are to given to theoretically interpret the so-called 
Decreasing Phenomenon and justify Algorithm 2. Theorem \ref{thm09} implies that the 
parameter $r_1 $ of stage 1 should be updated by decreasing to a small number 
when the parameter $r_2 $ of stage 2 is fixed.

\begin{theorem}
\label{thm09}
$\tilde {r}_1 = \arg \mathop {\min }\limits_r f_2 
(r\vert r_2 ) \leq \arg \mathop {\min }\limits_r f_1 (r) = r_1 $ where $r_2 > 
r_1 $.
\end{theorem}
\begin{proof}

$\because\tilde {r}_1 = \arg \mathop {\min }\limits_r f_2 (r\vert r_2 )$ and $r_1 = \arg \mathop {\min }\limits_r f_1 (r)$,

$\therefore\left. \frac{df_2 (r\vert r_2 )}{dr} \right|_{r = \tilde {r}_{ 1 }}  - 
\left. \frac{df_1 (r)}{dr} \right|_{r = r_{ 1 }} = 0$.

$\because f_1 (r) = p_1 (r)(r+c) + (1 - p_1 (r))(T+2c),$

$\therefore{f}'_1 (r_1 ) = p_1 (r_1 ) + {p}'_1 (r_1 )(r_1-T-c).$
 
$\because f_2 (r\vert r_2 ) = p_1 (r)(r+c) + (1 - p_1 (r))[p_2 (r_2 \vert r)(r_2+2c)+(1 - p_2 (r_2 \vert r))(T+3c)],$

$\therefore{f}'_2 (r\vert r_2 ) = p_1 ^\prime (r)(r - r_2-c ) + p_1 (r)$ (see (\ref{eq59})).

Now investigate the value of ${f}'_2 (\tilde {r}_1 \vert r_2 ) - {f}'_1 (r_1 
)=p_1 ^\prime (\tilde {r}_1 )(\tilde {r}_1 - r_2-c ) - {p}'_1 (r_1 )(r_1 - T-c)+ (p_1 (\tilde {r}_1 ) - p_1 (r_1 ))$ if $\tilde {r}_1 > r_1 $ is true:

$\because\tilde {r}_1 > r_1 $ is assumed,

$\therefore 0 < p_1 ^\prime (\tilde {r}_1 ) < {p}'_1 (r_1 ), 0 > \tilde {r}_1-r_2 -c> r_1-T-c, p_1 (\tilde {r}_1 ) > {p}_1 (r_1 ) > 0,$

$\therefore p_1 ^\prime (\tilde {r}_1 )(\tilde {r}_1 - r_2-c ) > {p}'_1 (r_1 )(r_1 - T-c), p_1 (\tilde {r}_1 ) > p_1 (r_1 ),$

$\therefore {f}'_2 (\tilde {r}_1 \vert r_2 ) - {f}'_1 (r_1 ) > 0$.

This contradicts ${f}'_2 (\tilde {r}_1 \vert r_2 ) - {f}'_1 (r_1 ) = 0$. So 
$\tilde {r}_1 > r_1 $ is wrong and $\tilde {r}_1 \leq r_1 $ is true. \qed
\end{proof}

As a lemma of Theorem \ref{thm09}, we have the following theorem:

\begin{theorem}
\label{thm10}
If $(r_1^{i\ast } ,...,r_i^{i\ast } ) = 
\arg \mathop {\min }\limits_{r_1 ,...,r_i } f_i (r_1 ,...,r_i )$ and 
$(r_1^{\ast (i + 1)} ,...,r_i^{\ast (i + 1)} ,r_{i + 
1}^{\ast (i + 1)} ) = \arg \mathop {\min }\limits_{r_1 ,...,r_i ,r_{i + 1} } 
f_i (r_1 ,...,r_i ,r_{i + 1} )$, then $r_j^{\ast (i + 1)} \leq r_j^{\ast i},$
$j = 1,...,i$.
\end{theorem}

As a generalized version of Theorem \ref{thm09}, Theorem \ref{thm10} tells that once a new 
stage $i + 1$ is added, all the optimal parameters of the existing stages 
$1,\ldots,i$ should be updated and decreased so that the 
computation cost is minimized.

\begin{theorem}
\label{thm11}
If $\tilde {r}_2 < r_2$, then $\tilde {r}_1 = \arg 
\mathop {\min }\limits_r f_2 (r\vert \tilde {r}_2 ) \leq \arg \mathop {\min 
}\limits_r f_2 (r\vert r_2 ) = r_1 $.
\end{theorem}
\begin{proof}

$\because\tilde {r}_1 = \arg \mathop {\min }\limits_r f_2 (r\vert \tilde{r}_2 )$ and $r_1 = \arg \mathop {\min }\limits_r f_2 (r|r_{2})$,

$\therefore\left. {\frac{df_2 (r\vert \tilde{r}_2 )}{dr}} \right|_{r = \tilde {r}_{1}} - 
\left. {\frac{df_2 (r|r_2)}{dr}} \right|_{r = r_{_1 }} = 0$ is true.

Now investigate the value of ${f}'_2 (\tilde {r}_1 
\vert \tilde {r}_2 ) - {f}'_2 ( {r}_1 \vert {r}_2 )$
$ = [p_1 ^\prime (\tilde {r}_1 )(\tilde {r}_1 - \tilde {r}_2-c ) - {p}'_1 (r_1 )(r_1 - r_2-c)]  + [p_1 (\tilde {r}_1 ) - p_1 (r_1 )]$ if $\tilde {r}_1 > r_1 $ is true:

$\because\tilde {r}_1 > r_1 $ is assumed,

$\therefore 0 < p_1 ^\prime (\tilde {r}_1 ) < {p}'_1 (r_1 ),
0 > \tilde {r}_1 - \tilde {r}_2-c > r_1 - r_2-c,
p_1 (\tilde {r}_1 ) > p_1 (r_1 ) > 0,$

$\therefore p_1 ^\prime (\tilde {r}_1 )(\tilde {r}_1 - \tilde {r}_2-c ) > {p}'_1 (r_1 )(r_1 - r_2-c), p_1 (\tilde {r}_1 ) > p_1 (r_1 ),$

$\therefore {f}'_2 (\tilde {r}_1 
\vert \tilde {r}_2 ) - {f}'_2 ( {r}_1 \vert {r}_2 ) > 0.$

This contradicts ${f}'_2 (\tilde {r}_1 
\vert \tilde {r}_2 ) - {f}'_2 ( {r}_1 \vert {r}_2 ) = 0$. So 
$\tilde {r}_1 > r_1 $ is wrong and $\tilde {r}_1 \leq r_1 $ is true. \qed
\end{proof}

\begin{theorem}
\label{thm12}
If $\tilde {r}_1 < r_1,$ then $\tilde {r}_2 = \arg 
\mathop {\min }\limits_{\tilde {r}_1 < r < T} f_2 (r\vert \tilde {r}_1 ) \leq 
\arg \mathop {\min }\limits_{r_1 < r < T} f_2 (r\vert r_1 ) = r_2 $.
\end{theorem}
\begin{proof} {See Appendix A.} \qed
\end{proof}

%
%
%
%
%
%
%

Theorems \ref{thm09}, \ref{thm11} and \ref{thm12} can be extended to multi-stage cascade. Correspondingly, Decreasing Phenomenon can be generalized to 
Generalized Decreasing Phenomenon and Algorithm 2 can be generalized to 
Algorithm 3. 

\textit{Generalized Decreasing Phenomenon:} If the alternative optimization 
algorithm 3 is used to find the globally optimal solution $(r_1^{\ast i} 
,r_2^{\ast i} ,...,r_i^{\ast i} ) = \arg \mathop {\min }\limits_{r_1 
,...,r_i ,} f_i (r_1 ,...,r_i )$, then it holds that: 

(1) Once a new stage $i + 1$ is added, all the optimal parameters of the 
existing stages $1,\ldots ,i$ are updated and decreased so that the 
computation cost is minimized.

(2) Once the number of stages is fixed, the parameter for each stage 
decreases gradually as iteration proceeds.

\begin{algorithm}[!t]
\renewcommand{\algorithmicrequire}{\textbf{Input:}}
\renewcommand\algorithmicensure {\textbf{Output:} }
\caption{ Globally optimal multi-stage cascade learning.}
\begin{algorithmic}[1]
\REQUIRE ~~\\
Strong classifier $H({\rm {\bf x}}) = \sum_{i=1}^{T} {\alpha _i h_i ({\rm {\bf x}})} $ and its threshold $t$;\\
A set of true negative sub-windows$\{{\rm {\bf x}}\vert l({\rm {\bf x}}) = - 1\}$;
\ENSURE ~~\\ 
$(r_1^ * ,...,r_S^ * ) = \arg \mathop {\min }\limits_{r_1 ,...,r_S } = f_S (r_1 ,....,r_S )$ where $S$ is the number of stages in the final cascade structure;

\STATE Search the optimal solution $r_1^{\ast (1)} $ of $f_1 (r)$ for stage 1 in the range of $1 < r < T$: $r_1 = \arg \mathop {\min }\limits_{1 < r < T} f_1 (r)$. $f \leftarrow f_1 (r_1 )$;

\FOR{$i=2$ to $S$}
\STATE Initialize the upper bound $r_1^u ,...,r_{i - 1}^u $ of $r_1 ,...,r_{i - 1} $: $r_j^u \leftarrow r_j^{\ast (i - 1)} $ for $j = 1,...,i - 1$;

\STATE Initialize the upper bound $r_i^u $ of $r_i $ by finding $r_i^u = \arg \mathop {\min }\limits_{r_i \geq 2r_{i - 1}^{\ast (i - 1)} - r_{i - 2}^{\ast (i - 1)} ,} f_i \left(r_i \vert r_1^{\ast (i - 1)} ,...,r_{i - 1}^{\ast (i - 1)} \right)$ with the search range $r_i \geq 2r_{i - 1}^{\ast (i - 1)} - r_{i - 2}^{\ast (i - 1)} $ and $r_0^{\ast (1)} \buildrel \Delta \over = 0$. 

\STATE $f \leftarrow f_i (r_1^{\ast (i - 1)} ,...,r_{i - 1}^{\ast (i - 1)} ),~\tilde {f} \leftarrow f_i (r_1^u ,...,r_i^u )$.

\WHILE{$f - \tilde {f} > \varepsilon$}
\STATE $f \leftarrow \tilde {f}$;
\FOR{$j=0$ to $i$}
\STATE $r_j^\ast = \arg \mathop {\min }\limits_{r_i \leq r_j^u } f_i (r_j \vert r_k^u ,k \ne j)$;
\ENDFOR
\STATE $\tilde {f} \leftarrow f_i (r_1^\ast ,...,r_i^\ast ),~r_j^u \leftarrow r_j^\ast .$
\ENDWHILE
\STATE $r_j^{\ast i} \leftarrow r_j^u ,~j = 1,...,i$;
\ENDFOR
\RETURN $r_i^\ast \leftarrow r_i^{\ast S} $, $i = 1,...,S$.
\end{algorithmic}
\end{algorithm}

In Algorithm 3, $\tilde {f}$ is the objective function after a new stage $i$ 
is added while $f$ is the one before stage $i$ is added. That is , $f$ is 
the value of objective function when there are $i - 1$ stages. According to 
Theorem \ref{thm05}, when a new stage $i$ is to be added, the optimal solution $r_i $ 
can be searched by increasing $r_i $ from $2r_{i - 1}^{\ast (i - 
1)} - r_{i - 2}^{\ast (i - 1)} $ instead of $r_{i - 1}^{\ast (i - 1)} $. 
Because $2r_{i - 1}^{\ast (i - 1)} - r_{i - 2}^{\ast (i - 1)} $ is much 
larger than $r_{i - 1}^{\ast (i - 1)} $, the search efficiency is very high. 
The iteration in line 5 of Algorithm 3 stops if the difference between $f$ and 
$\tilde {f}$ is below a threshold $\varepsilon > 0$, which implies that the 
algorithm arrives at global minimum solution for $i$ stages. 

Fig. 10 shows the classification procedure of multi-stage iCascade where the partition points ($r_1$,...,$r_S$) are given by Algorithm 3. If the computation cost of classifying positive samples is neglected, the computation cost $f_S $ of iCascade can be estimated by

\begin{equation}
\label{eq60}
\begin{array}{l@{~}l}
f_S =& \sum\limits_{i = 1}^S {(r_i+ic) \left[ {\prod\limits_{j = 1}^i {(1 - p_{j - 1} 
(r_{j - 1}))} } \right]} p_i (r_i)\\
& +(T+(S+1)c)\left[ {\prod\limits_{j = 1}^{S+1} {(1 - p_{j-1 } (r_{j-1}))} } \right].
\end{array}
\end{equation}


\subsection{Threshold learning in iCascade}
Once the number of weak classifiers in each stage is determined by Algorithm 3, the parameters affecting the computation cost are the thresholds $t_i $, $i=1,\ldots ,S$. In this section, we give theorem and algorithm for setting the thresholds ($t_1$,..., $t_S$). 

\begin{figure}
\label{FigiCascade}
\centering
\includegraphics{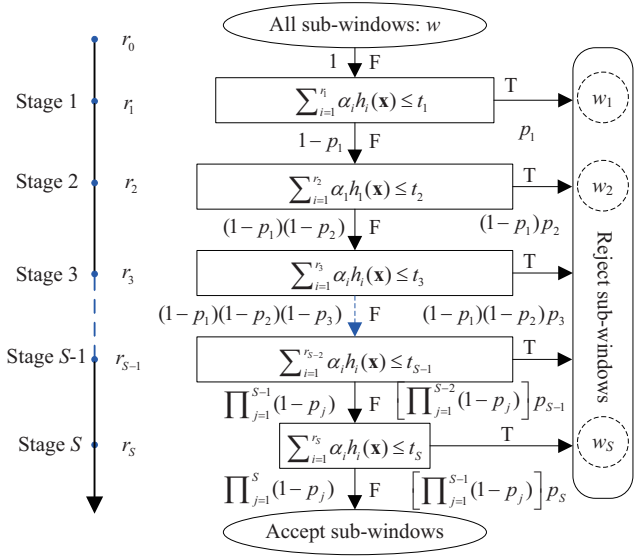}
\caption{The classification procedure of the multi-stage iCascade algorithm.}
\end{figure}

Theorem \ref{thm13} tells that the computation cost $f_S $ monotonically decreases with $t_i $ and $p_i (t_i )$, $i = 1,...,S$. So the 
computation cost can be reduced by decreasing the thresholds under the 
constraint of minimum-acceptable detection rate. 

\begin{theorem}
\label{thm13}
$f_S $ monotonically decreases with $t_i $ 
and $p_i (t_i )$, $i = 1,...,S$.
\end{theorem}
\begin{proof} See Appendix B. \qed
\end{proof}

If the detection rate $D = 1$ (i.e., all the positive training samples 
are correctly classified) is the constraint, then the optimal 
threshold $t_i^ * $ can be expressed as:
\begin{equation}
\label{eq61}
t_i^\ast = \arg \mathop {\min } t_i ,\;\;\mbox{s.t.}\;\;d(t_i 
) = 1,\quad i = 1,...,S.
\end{equation}
which guarantees $D = \prod\nolimits_{i = 1}^S {d(t_i^ * )} = 1$. In (\ref{eq61}), 
$\;d(t_i )$ is the detection rate of stage $i$ defined by: 

\begin{equation}
d(t_i ) = \frac{\sum\limits_{\bf{x}} {I(\sum\limits_{j = 1}^{r_i} 
{\alpha _j h_j ({\rm {\bf x}})} > t)} }{\sum\limits_{\rm {\bf x}} {I(l({\rm 
{\bf x}}) == 1)} }.
\end{equation}

It is challenging to choose the optimal thresholds if the expected detection 
$D < 1$. It is well known that the detection rate $D$ of the system is the 
product of the detection rate $d(t_i )$ of each stage. A popular way to set 
$d(t_i )$ is 

\begin{equation}
\label{eq63}
d(t_i ) = D^{1 / S},\;i = 1,...,S.
\end{equation}

However, when the number of stages of iCascade is very large, it holds that $d(t_i ) \approx 1$. 
Such high $d(t_i )$ makes the threshold $t_i $ very large and the corresponding 
computation cost is very large. 

To deal with the above problem, we propose to use Algorithm 4 for threshold 
learning. The initial thresholds are chosen by (\ref{eq63}) guaranteeing the 
detection rate $D$ being 1. The corresponding initial computation cost is 
denoted by $f_S $. The main issue is to select which stage to decrease its 
initial threshold by a small step $\Delta t_i$. In our algorithm, the 
derivative ${f}'_S$ of the computation cost $f_S $ against detection rate 
$D$ is computed by

\begin{equation}
{f}'_S(i) \approx \Delta f_S / \Delta D_i ,
\end{equation}

\noindent
where $\Delta D_i $ is the variation of the system detection rate. Note 
that the variation $\Delta D_i $ is caused by changing $t_i $ to $t_i - 
\Delta t_i$ while the thresholds $t_k $ of other stages (i.e., $k \ne i$) remain 
unchanged. 

The stage $j$ with the largest derivative is selected and its threshold $t_j 
$ is then decreased by the small step $\Delta t_j$:

\begin{equation}
j = \arg \mathop {\max }\limits_i \frac{\Delta f_S }{\Delta D_i },
\end{equation}

\begin{equation}
t_j \leftarrow t_j - \Delta t_j,
\end{equation}

\noindent
with the thresholds of the stages (i.e., $i \ne j$) unchanged.

Re-compute the computation cost $f_S $ and detection rate $D$ when $t_j $ 
is updated:

\begin{equation}
f_S \leftarrow f_S - \Delta f_S,
\end{equation}

\begin{equation}
D \leftarrow D - \Delta D_j.
\end{equation}

The step $\Delta t_i$ is small enough to keep the detection rate $D$ smaller than 
the target detection rate $D_o $. 

As shown in Algorithm 4, the iteration of choosing the most important stage 
$j = \arg \mathop {\max }\limits_i \Delta f_{S} / \Delta D_i $, updating 
its threshold $t_j \leftarrow t_j - \Delta t_j$ and corresponding 
computation cost $f_S \leftarrow f_S - \Delta f_S $ and detection rate $D 
\leftarrow D - \Delta D_j$ runs until the updated detection rate $D$ is below 
the expected detection rate $D_o $. 

\begin{algorithm}[!t]
\renewcommand{\algorithmicrequire}{\textbf{Input:}}
\renewcommand\algorithmicensure {\textbf{Output:} }
\caption{ Threshold learning algorithm for iCascade.}
\begin{algorithmic}[1]
\REQUIRE ~~\\
Expected detection rate $D_o $;\\
Positve and negative training samples;\\
Strong classifiers $H({\rm {\bf x}}) = \sum\nolimits_{i = 1}^T {\alpha _i h_i ({\rm {\bf x}})} $;
\ENSURE ~~\\ 
The optimal thresholds $t_i $ of all the $S$ stages;

\STATE Initialize the thresholds $t_i $ for each stage by $t_i = \arg \mathop {\min } t_i$ , $\mbox{s.t.}\;\;d(t_i ) = 1$, $i = 1,...,S$ so that the system detection rate $D=1$;

\STATE Corresponding to the initial thresholds, the initial computation cost of the system is computed by (\ref{eq60}) and denoted by $f_S$;

\REPEAT
\STATE For each stage, compute the approximation of the derivative $\Delta f_S / \Delta D_i $ of the computation cost $f_S $ against detection rate $D$. The variations $\Delta f_S $ and $\Delta D_i $ are caused by changing $t_i $ to $t_i - \Delta t_i$ while the thresholds $t_k $ of other stages (i.e., $k \ne i$) remain unchanged;

\STATE From all the $S$ stages, choose the stage $j$ with largest derivative $j = \arg \mathop {\max }\limits_i \Delta f_S / \Delta D_i $. Then decrease the threshold $t_j $ of the stage $j$ by a small step $\Delta t_j$ : $t_j \leftarrow t_j - \Delta t_j$;

\STATE Update the computation cost $f_S $ and detection rate $D$: $f_S \leftarrow f_S - \Delta f_S $, $D \leftarrow D - \Delta D_j$;
\UNTIL{$D\le D_o$}

\RETURN the updated thresholds $t_i $ of all the $S$ stages.
\end{algorithmic}
\end{algorithm}

\section{EXPERIMENTAL RESULTS}

\subsection{Experimental Setup}
The classical cascade learning algorithms of Viola and Jones (VJ)(a.k.a., Fixed Cascade) \cite{Viola_face_detection_IJCV_2004}, Recyling Cascade \cite{Brubaker_Design_IJCV_2008} and Recyling \& Retracting Cascade \cite{Brubaker_Design_IJCV_2008} are compared with 
the proposed iCascade algorithm. The testing dataset is the
standard MIT-CMU frontal face databaset \cite{Rowley_NNFD_PAMI_1998}, \cite{Viola_face_detection_IJCV_2004}. The positive training dataset consists of about 20000 normalized face 
images of size 20$\times $20 pixels. 5000 non-face large images are collected from web sites to generate negative training dataset. Both of the positive and negative training images can be downloaded from https://sites.google.com/site/yanweipang/publica. 

In addition, the intermediate results 
demonstrating the correctness of the proposed theorems are given in Section VI.B.

A strong classifier $H({\rm {\bf x}}) = \sum\nolimits_{i = 1}^T {\alpha _i 
h_i ({\rm {\bf x}})} $ is considered input of iCascade. The strong 
classifier is obtained by standard AdaBoost algorithm without designing of 
cascade structure.

\subsection{Intermediate Results of iCascade}
Some intermediate results are shown in this section. These results show the rationality of the assumptions and the correctness of the proposed theorems.
\subsubsection{Local-Minimum Based Cascade}
In Section III, the regular strong AdaBoost classifier is divided into $H_L(\textbf{x},r)$ and $H_R(\textbf{x},r)$ to 
reject some negative sub-windows earlier, and the key problem is to 
determine an optimal $r$ to minimize the computation cost. To solve this 
problem, it is necessary to reveal the relationship between $r$ and the 
negative rejection rate $p$. 

In this part, with the training dataset described in Section VI.A, we train a regular strong AdaBoost classifier and split it 
into two parts by $r$, which varies $1$ to $T$. In the case that detection rate is fixed at 1, Fig. 11 shows that the negative rejection rate $p$ increases with $r$. $p$ first grows 
quickly from 0 to 0.96 when $r$ changes from 1 to a small value $r^\ast=80$, 
and then becomes stable when $r$ is larger than $r^\ast $ . Thus, we 
can model $p(r)$ by combining of two linear functions: $p_1 (r) = 0.012r$ with $r 
< r^\ast $ and $p_2 (r) = 1$ with $r \ge r^\ast $. Fig. 11 demonstrates the rationality of (\ref{eq16})-(\ref{eq25}). Fig. 12 shows that the computation cost $f$ first decreases and then increases with $r$, and the unique minimum is nearby $r^\ast $. Fig. 12 experimentally proves the correctness of Theorem \ref{thm01} and Theorem \ref{thm02}.
\begin{figure}[!t]
\label{FigExperimentNegativeReject}
\centering
\includegraphics[width=3.2in]{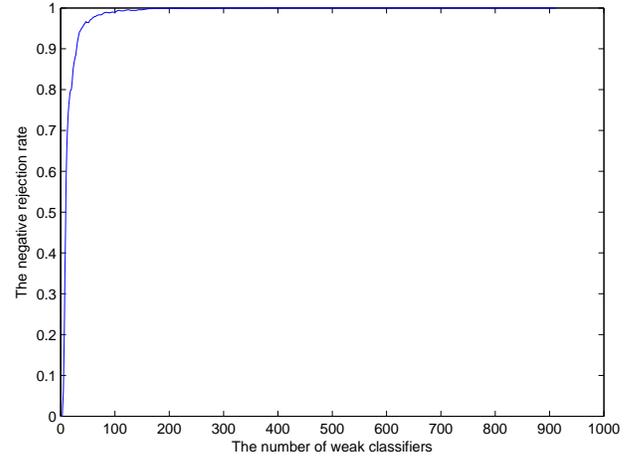}
\caption{The negative rejection rate $p(r)$}
\end{figure}

\begin{figure}[!t]
\label{FigExperimentCompuatation}
\centering
\includegraphics[width=3.2in]{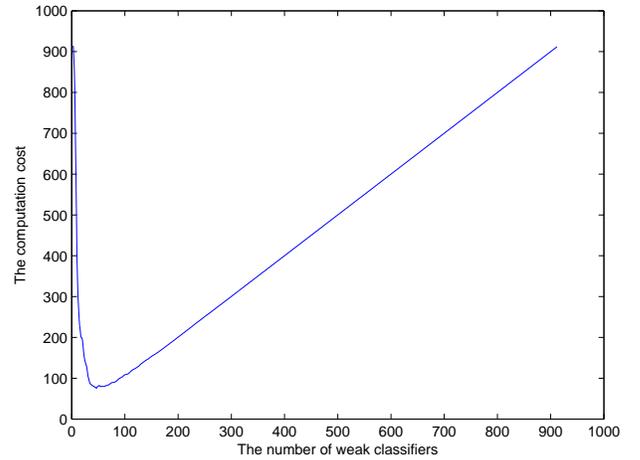}
\caption{The computation cost $f(r)$}
\end{figure}

When we split the regular strong AdaBoost classifier into $H_L ({\rm {\bf x}},r_1 )$ and $H_R ({\rm {\bf x}},r_1 )$, the sub-windows not rejected by stage 1 are fed to 
stage 2. Then we can divide $H_R ({\rm {\bf x}},r_1 )$ into two parts to form a
2-stage cascade. In this process, we should know some properties of the 
negative rejection rate of stage 2 (i.e., $p(r\vert r_1 ))$. Fig. 13 shows how $p(r\vert r_1 )$ changes with $r$, where the curves of $p(r|r_1)$ when $r_1 = 10$ and $r_1 = 30$ are given, respectively. $p(r|r_{1})$ has the similar characteristics with $p(r)$. Fig. 14 shows how the derivative curves of $p(r\vert r_1 )$ change with $r$. Obviously, when $\widetilde{r_1 } < r_1 $, $p(r_2 \vert \widetilde{r}_1 ) > p(r_2 \vert r_1 )$ and ${p}'(r_2 \vert \widetilde{r}_1 ) < {p}'(r_2 \vert r_1 )$.  Fig. 13 and 14 directly support the correctness of (\ref{eq39})-(\ref{eq43}).
\begin{figure}
\label{FigExperimentProperties}
\centering
\includegraphics[width=3.2in]{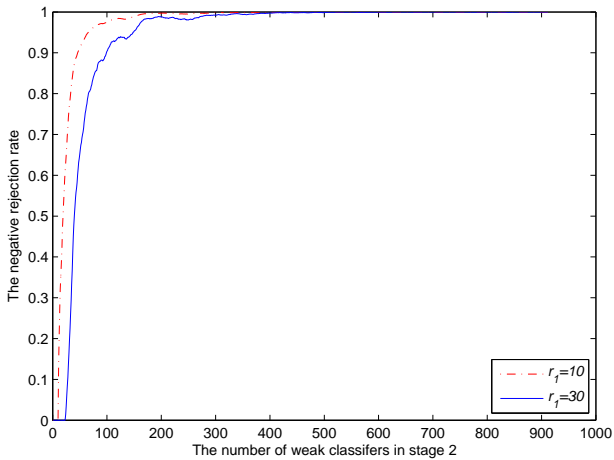}
\caption{Some properties of $p(r_2 \vert r_1 )$}
\end{figure}
\begin{figure}
\label{FigExperimentDerivative}
\centering
\includegraphics[width=3.2in]{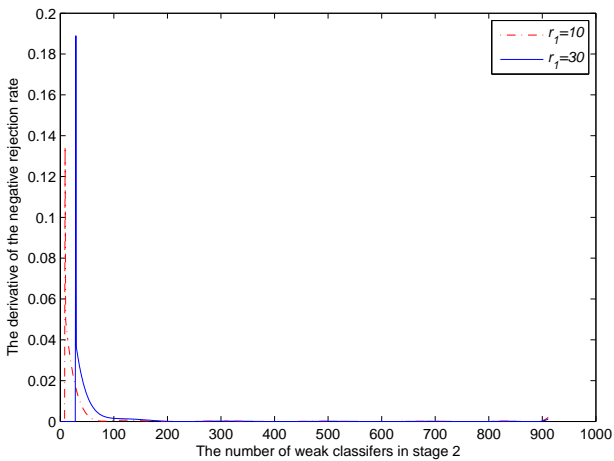}
\caption{The derivative of $p(r_2 \vert r_1 )$}
\end{figure}

We use the local-minimum based multi-stage cascade learning method (see Fig. 4) to train an 8-stage cascade classifier. Table 1 shows how the computation cost 
$f$ changes with the number of stages. The computation cost first decreases 
quickly and then becomes stable. This phenomenon can be understood 
easily, because the first few stages can reject the most part of the 
sub-windows, and then only some small part of the sub-windows arrive at last 
few stages which don't produce much computation cost.

\begin{table*}[!t]
\centering
\renewcommand{\arraystretch}{1.3}
\caption{Computation cost $f$ of the algorithm in Fig. 4 varies with the number $s$ of stages}
\begin{tabular}
{|c|c|c|c|c|c|c|c|c|}
\hline
$s$& 1& 2& 3& 4& 5& 6& 7& 8 \\
\hline
$f$& 74.42& 52.87& 52.16& 52.15& 52.14& 52.14& 52.14& 52.14 \\
\hline
\end{tabular}
\end{table*}

\subsubsection{Joint-Minimum Based Cascade}
In the local-minimum based multi-stage cascade(i.e., Fig. 4), it seeks an optimal $r_i $ on the condition that $(r_1 ,...,r_{i - 1} )$ are known and fixed, so $(r_1 ,...,r_{i - 1} ,r_i )$ can't be jointly optimal for minimizing the computation cost $f(r_1 ,...,r_i )$ where not only $r_{i}$ but also $r_{1},\ldots,r_{i-1}$ are variable. Thus, Algorithm 3 is proposed to train the joint-minimum based multi-stage cascade.

Fig. 15 shows the iteration process of Algorithm 3. 
The number 48 on the top blue line is $r_1^{*(1)}=\arg\min f_1(r)$, which is the result of line 1 of Algorithm 3. The right number 172 on the top red line is $r_2^u=\arg\min f_2(r_2|r_1^{*(1)})=172$ (see line 3 of Algorithm 3). Obviously $r_2^u=172$ is the solution of local-minimum based optimization. The number 17 and 82 on the second blue line are solutions of joint-minimum based optimization (i.e., line 13 of Algorithm 3). Generally, the right most number on each red line is the upper bound $r_i^u$ of Algorithm 3, and the number on each blue line are the solutions of joint-minimum based optimization $r_j^{*i},j=1,\ldots,i$.
The generalized decreasing phenomenon can be found from Fig. 15. For example, $r_1$ decreases from 48 to 17, 12; $r_2$ decreases from 172 to 82, 45, 33; $r_3$ decreases from 180 to 172, 69, 60. Table 2 gives the computation cost of the cascade corresponding to Fig. 15. In Table 2, $f_{local}(1)=74.42$ is the computation cost $f_1(r_1)$ with $r_1=48$, and $f_{local}(2)=52.87$ is equal to $f_2(r_2|r_1)$ with $r_1=48$ and local optimization solution is $r_2=172$. Generally, $f_{local}(i)$ means the computation cost $f(r_i|r_1,\ldots,r_{i-1})$. In Table 2, $f_{joint}(i)$ is the computation cost $f(r_1,\ldots,r_i)$ of the proposed joint-minimum algorithm where $r_1,\ldots,r_i$ are all unknown and $i$ is the total number of the stages of the cascade. Note that $f_{joint}(1)=f_{local}(1)$, because there is only one stage in the cascade. However, $f_{joint}(2)=37.83$ and $f_{joint}(3)=26.67$ are much less than $f_{local}(2)=52.87$ and $f_{local}(3)=31.24$, respectively.

\begin{figure}[!t]
\label{FigExperimentGeneralized}
\centering
\includegraphics[width=3.2in]{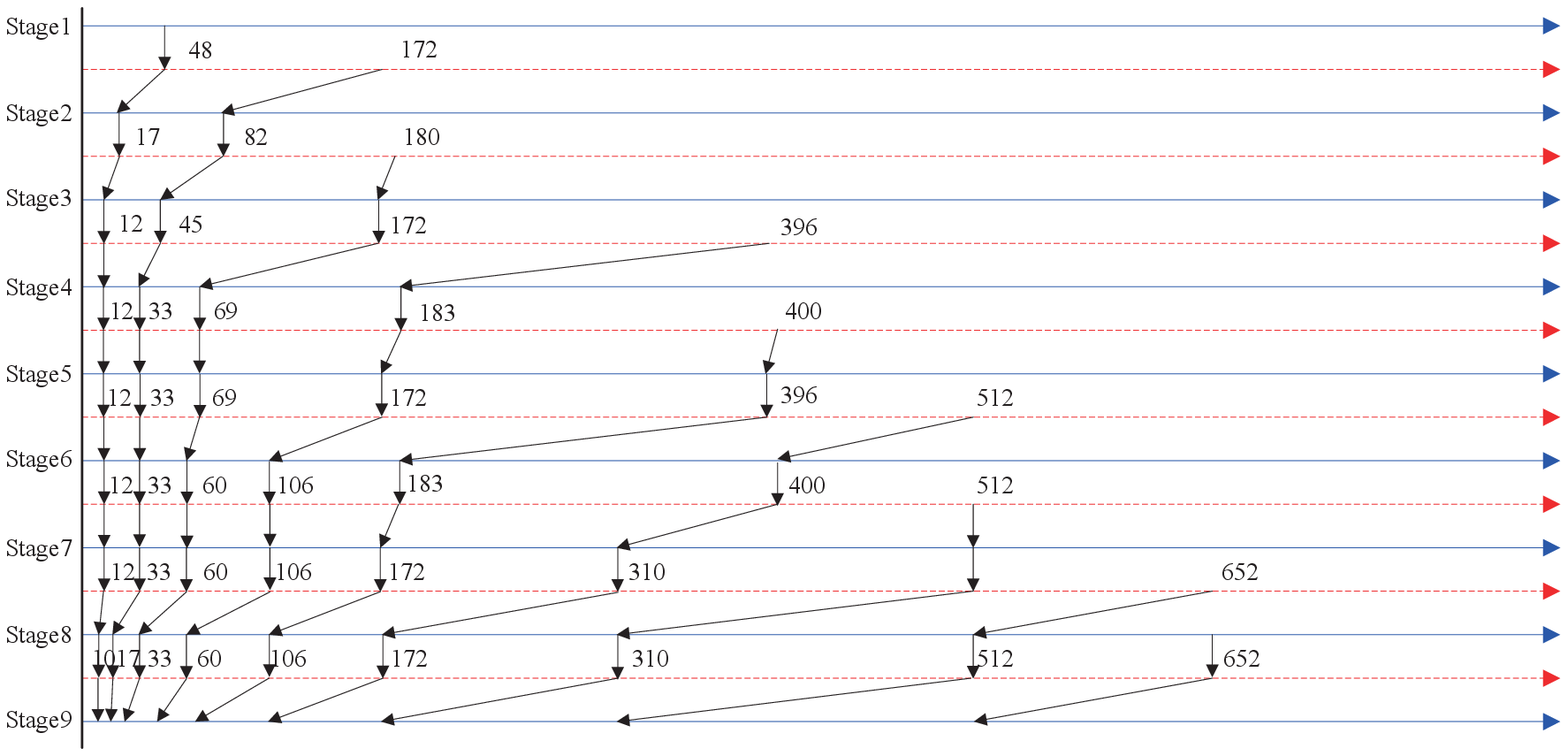}
\caption{Generalized decreasing phenomenon in joint-minimum based multi-stage cascade}
\end{figure}

\begin{table*}[!t]
\centering
\renewcommand{\arraystretch}{1.3}
\caption{Computation cost of the cascade trained by Algorithm 3}
\begin{tabular}
{|c|c|c|c|c|c|c|c|c|c|}
\hline
stage number $i$& 1& 2& 3& 4& 5& 6& 7& 8& 9 \\
\hline
$f_{local}(i)$ & 74.42& 52.87& 31.24& 26.00& 22.16& 21.99& 21.26& 21.19& 18.98 \\
\hline
$f_{joint}(i)$& 74.42& 37.83& 26.67& 22.70& 22.06& 21.31& 21.20& 18.99& 18.38 \\
\hline
\end{tabular}
\end{table*}

To compare the joint-minimum Algorithm 3 with the local-minimum algorithm (see Fig. 4), we visualize $f_{joint}$ in Table II and $f$ in Table I in Fig. 16. With the number of stages increasing, the computation costs decrease. But the difference is that the joint-minimum based algorithm decreases more quickly than the local-minimum algorithm. For example, when the numbers of stages are 3 and 8, the computation costs  of the joint-minimum and local-minimum algorithms are (26.67 and 18.99) and (52.16 and 52.14), respectively. In summary, Fig. 16 demonstrates the advantage and importance of the proposed joint-minimum optimization algorithm.

\begin{figure}
\label{FigExperimentComparisonLocalandJoint}
\centering
\includegraphics[width=3.2in]{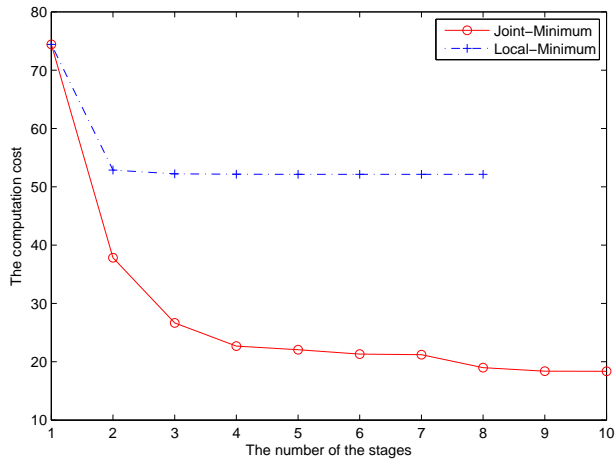}
\caption{Comparison of the computation cost between local-minimum based multi-stage cascade and joint-minimum based one}
\end{figure}

\subsubsection{Threshold learning}
 The thresholds $t_i,i=1,\ldots,T$ affect the computation cost of iCascade. Algorithm 4 gives the iteration process to choose the threshold of each stage for iCascade. Note that the variation $\Delta D_i$ of detection rate
 is obtained by changing $t_i$ to $t_i - \Delta t_i$. As $ \Delta t_i$ gradually decreases,
 the detection accuracy increases whereas the training
 time drastically grows. A set of $t_i$ is evaluated. We find that
 the performance is stably good if $t_i \leq 0.02$. As a tradeoff,
 $t_i = 0.01$ is empirically employed. Fig. 17 shows how the computation cost updates in the iteration process of the first 20 stages' thresholds. It can be seen that the computation cost significantly decreases with the iteration. In addition, Fig. 17 shows the convergence of the proposed threshold learning algorithm. Fig. 17 supports the correctness of Theorem \ref{thm13}.
\begin{figure}
\label{FigExperimentComputationUpdate}
\centering
\includegraphics[width=3.2in]{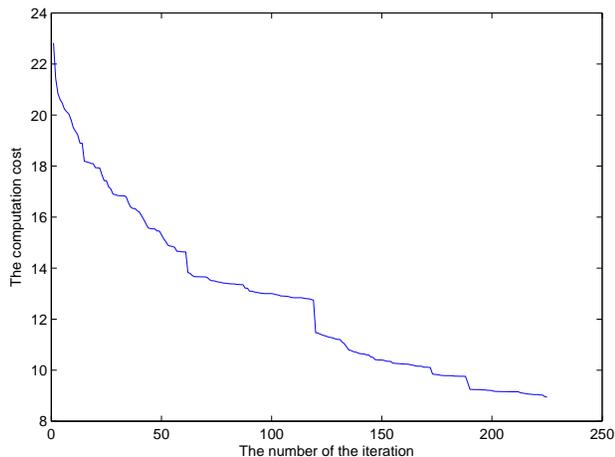}
\caption{The computation cost decreases with updation of threshold $t_i$}
\end{figure}

\subsection{Comparison With Other Algorithms}
In this section, we compare iCascade with some 
other algorithms, including Fixed Cascade \cite{Viola_face_detection_IJCV_2004}, Recyling Cascade \cite{Brubaker_Design_IJCV_2008} and Recyling \& 
Retracting Cascade \cite{Brubaker_Design_IJCV_2008}. 

Fixed Cascade is proposed by Viola and Jones. "Fixed" means that the 
detection rate $d_i$ and the false detection rate $f_i $ of each stage 
is same and fixed, respectively. If the target detection rate of the 
cascade is $D$, the target rejection rate is $F$ and the number of the 
stages is $N$, then $d_i = D^{1 / N}$ and $f_i = F^{1 / N}$. In Recyling Cascade, the score from the previous strong classifier stages serves as a starting point for the score of the new strong classifier stage. 
The benefit of Recyling Cascade is the reduction of the number of the weak classifiers in the strong classifier stages and the reduction of the computation cost. The side-effect of Recyling Cascade is that the last stage of cascade can serve as an accurate strong classifier. Recyling \& Retracting Cascade chooses a threshold after each weak classifier of the strong classifier is produced by Recyling Cascade to reject some negative sub-windows. To set these thresholds, it evaluates each score on the set of the positive examples and chooses the minimum score as the threshold so that all the positive examples in the set can pass all the weak classifiers.

These algorithms are evaluated on the standard MIT-CMU frontal face database \cite{Rowley_NNFD_PAMI_1998}, \cite{Viola_face_detection_IJCV_2004}, 
which consists of 125 grayscale images containing 483 labeled frontal faces.
If the detected rectangle and the ground-truth rectangle are at least 50 
percent of overlap, we call the detected rectangle a correct detection. The number of average features per window is used to represent the computation cost. Fig. 18 reflects the computation cost of different algorithms (i.e., iCascade, Recyling Cascade and Retracting \& Recyling Cascade) as a function of image location. The number of the average features used in a sliding window is accumulated to the center pixel of this sliding window. After detection, the value of each pixel is normalized to 0-255. The larger the value is, the greater the computation cost is, and the greater the probability that the face exists here is. It can be observed that Fig. 18(d) (i.e., iCascade) is much darker and sparser than Fig. 18(b) and (c). The darkness and sparisity imply that iCascade consumes less computation cost than the other two algorithms.

Fig. 19 shows the average number of features applied per window of different methods at different expected detection rates $D_0$. For example, when the detection rate is 0.98, iCascade averagely uses 5.95 features, wheras Fixed Cascade, Recyling Cascade and Reyling \& Retracting Cascade use 22.84, 20.78 and 13.32, respectively. Fig. 20 shows the ROC of the different algorithms. The detection performance of different methods is no significant difference. From Fig. 19 and 20, we can conclude that iCascade has less computation cost with no loss of detection performance. 

\begin{figure}[!t]
\centering
\subfloat[Original image]{\includegraphics[width=1.5in]{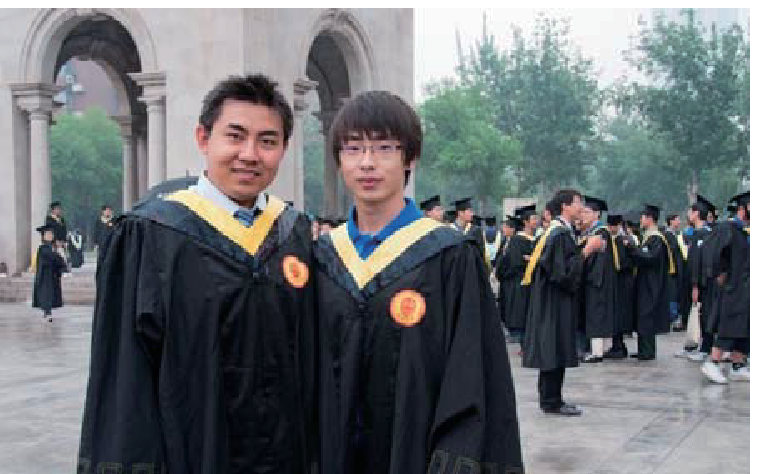}%
\label{fig_first_case}}
\hfil
\subfloat[Recyling response image]{\includegraphics[width=1.5in]{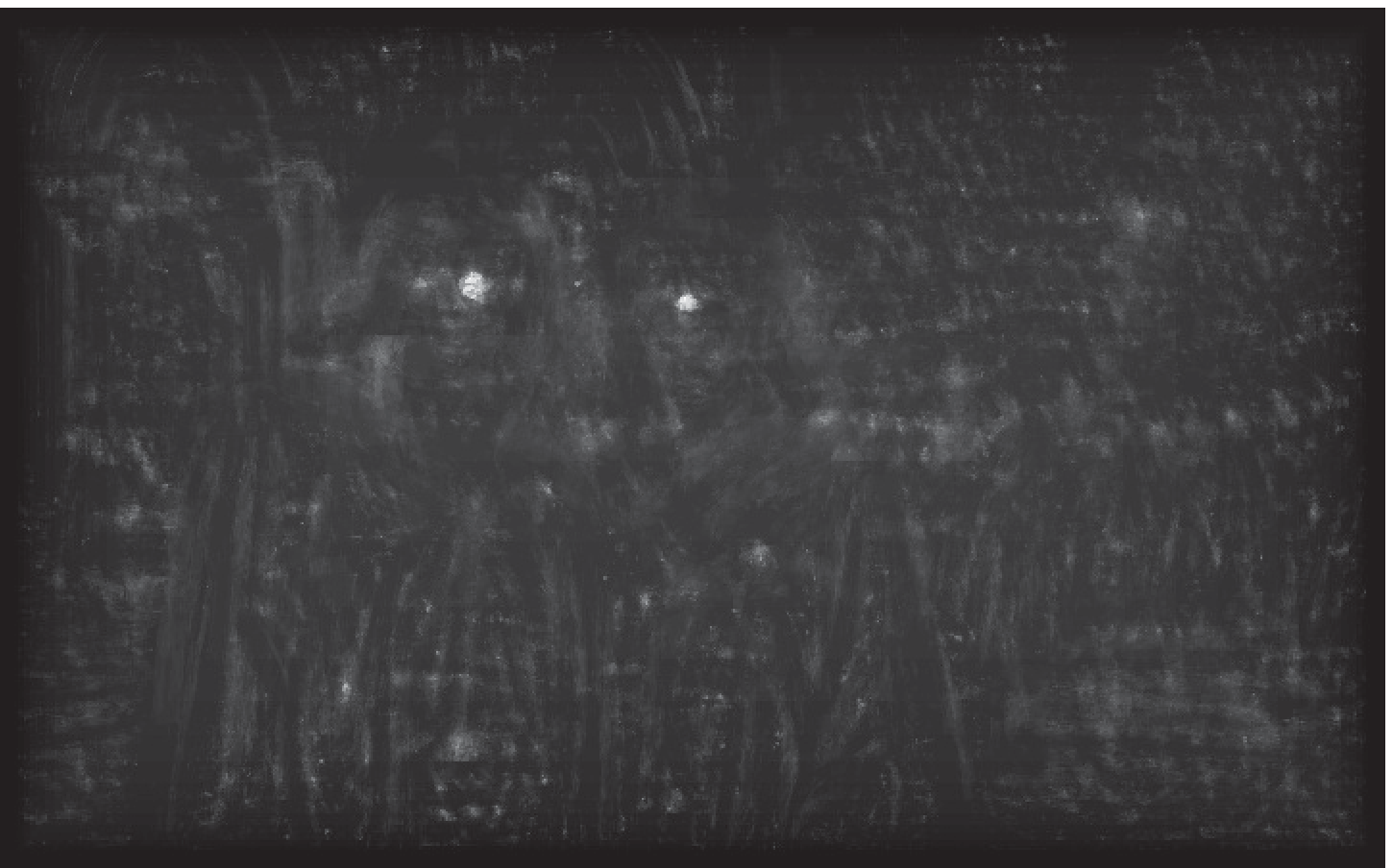}%
\label{fig_second_case}}
\hfil
\subfloat[Recyling \& Retracting response image]{\includegraphics[width=1.5in]{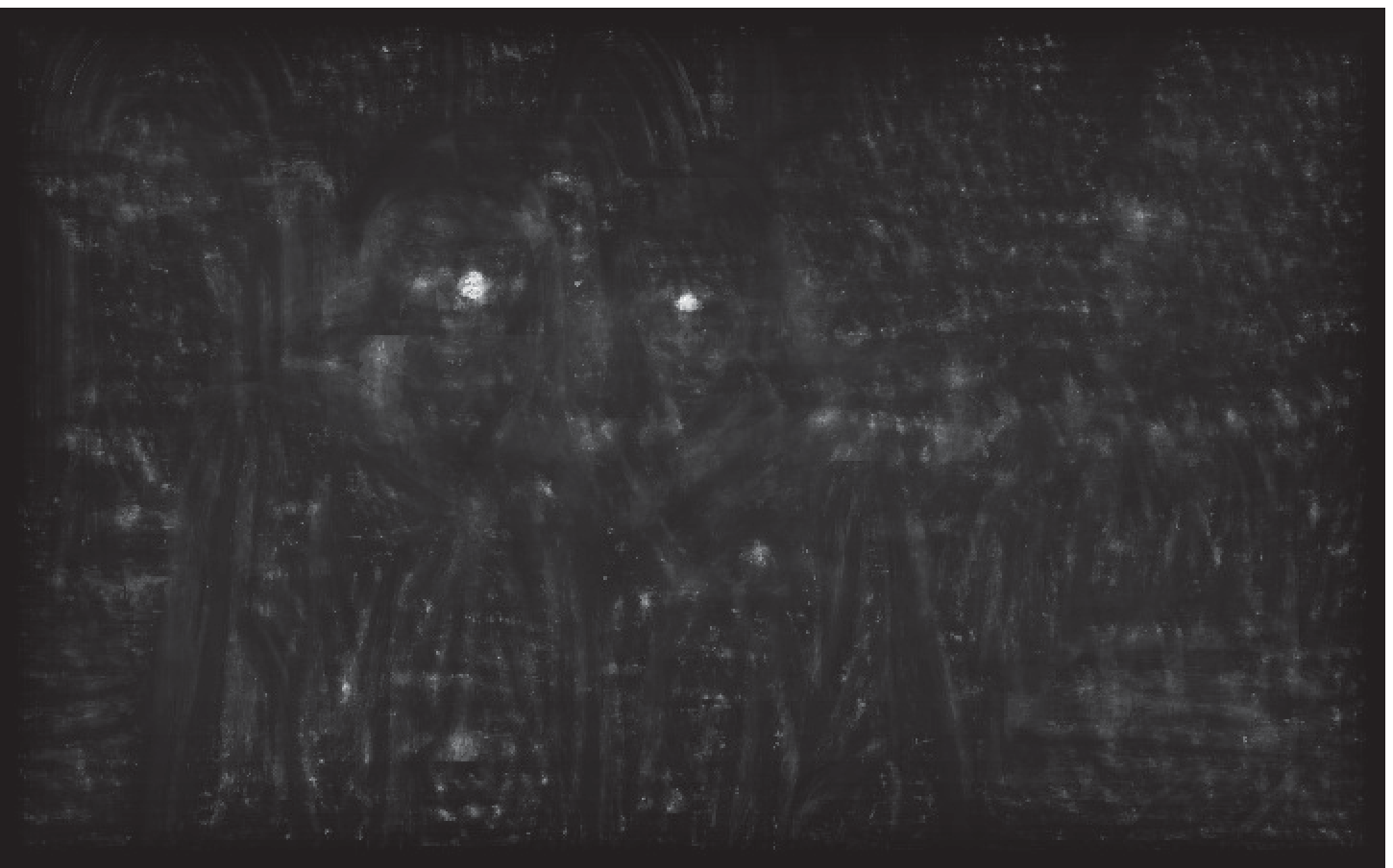}%
\label{fig_third_case}}
\hfil
\subfloat[iCascade response image]{\includegraphics[width=1.5in]{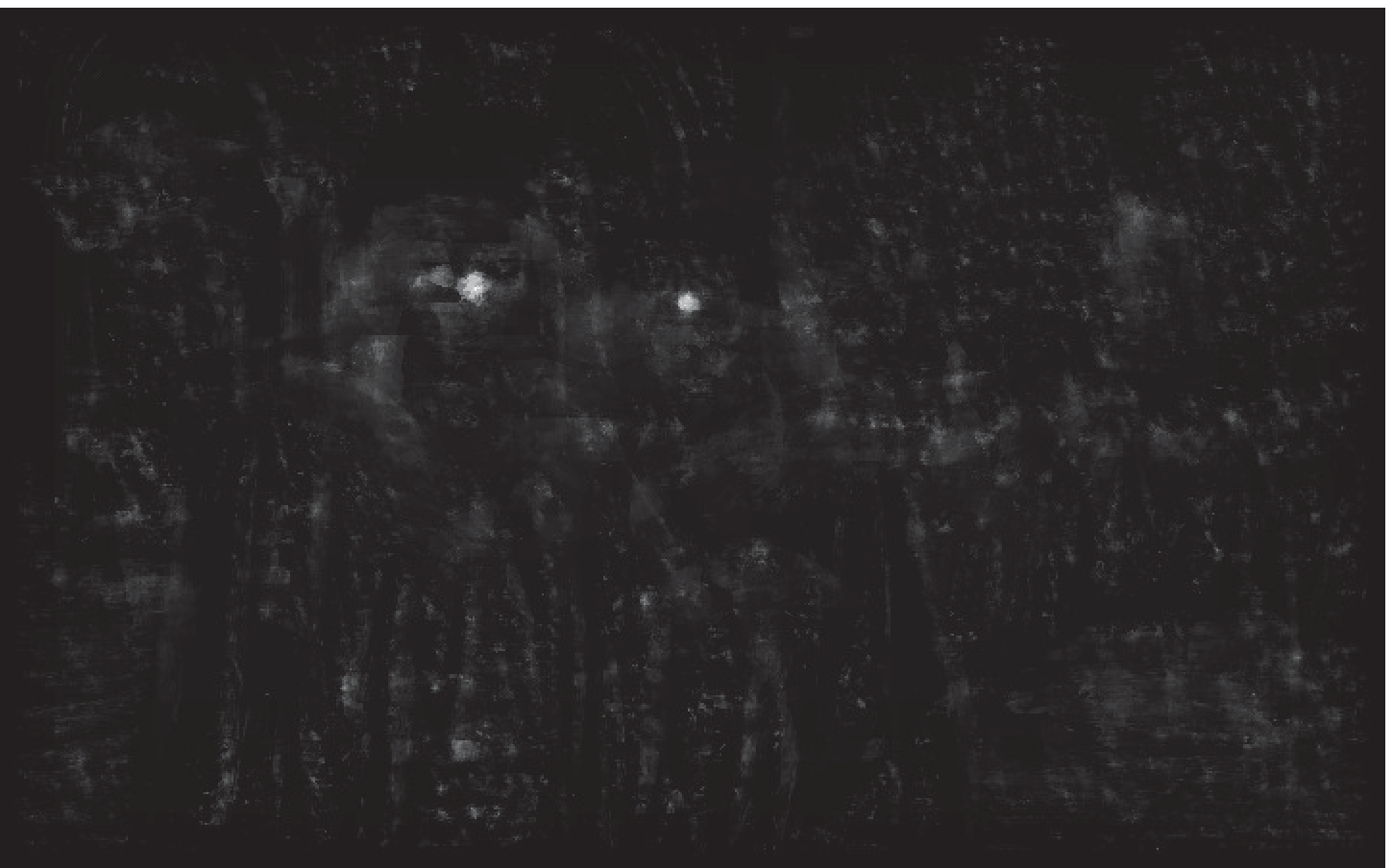}%
\label{fig_four_case}}
\caption{The computation cost shown as a function of image location.}
\label{fig_sim}
\end{figure}

\begin{figure}[!t]
\label{FigExperimentComparisonAlagorithms}
\centering
\includegraphics[width=3.2in]{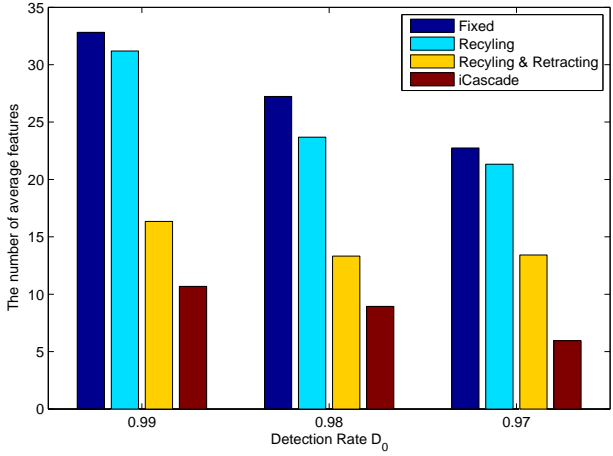}
\caption{Comparison of the computation cost between different algorithms}
\end{figure}

\begin{figure}[!t]
\label{FigExperimentROC}
\centering
\includegraphics[width=3.2in]{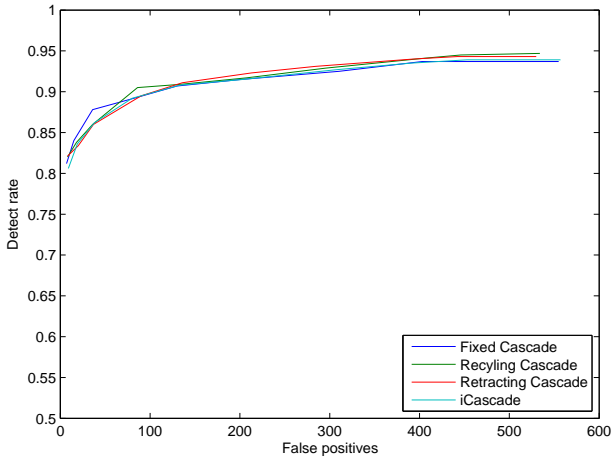}
\caption{ROC of different algorithms}
\end{figure}

\section{Conclusion}
In this paper, we have proposed to design a one-stage cascade structure by 
partitioning a strong classifier into left and right parts. Moreover, we 
have proposed to design a multi-stage cascade structure by iteratively 
partitioning the right parts. Solid theories have been provided to guarantee the 
existence and uniqueness of the optimal partition point with the goal of 
minimizing computation cost of the designed cascade classifier. Decreasing 
phenomenon has been discovered and theoretically justified for efficiently 
searching the optimal solutions. In addition, we have presented an effective algorithm for learning the optimal threshold of each stage classifier.


%
\section*{Appendix A: Proof of Theorem \ref{thm12}}
\begin{proof}
$\because\tilde {r}_2 = \arg \mathop {\min }\limits_r f_2 (r\vert \tilde {r}_1 )$ and 
$r_2 = \arg \mathop {\min }\limits_r f_2 (r\vert r_1 )$,

$\therefore\frac{df_2 (\tilde {r}_2 \vert \tilde {r}_1 )}{d\tilde {r}_2 } - \frac{df_2 (r_2 \vert r_1 )}{dr_2 } = 0$ is true.

Now investigate the value of ${f}'_2 (\tilde {r}_2 
\vert \tilde {r}_1 ) - {f}'_2 (r_2 \vert r_1 ) = \left[ {p_1 ^\prime (\tilde {r}_1 )(\tilde {r}_1 - \tilde {r}_2-c ) - p_1^\prime (r_1 )(r_1 - r_2-c )} \right] + [p_1 (\tilde {r}_1 ) - p_1 (r_1 )]$ if $\tilde {r}_2 > r_2 $ is true:

$\because\tilde {r}_2 > r_2$ and $\tilde {r}_1 < r_1$ are assumed,

$\therefore 0<p_1^\prime(r_{1})<p_1^\prime(\tilde{r}_{1}), \tilde {r}_1 - \tilde {r}_2-c < r_1 - r_2-c < 0,0<p_1(\tilde{r}_{1})<p_1(r_{1}),$

$\therefore {p_1 ^\prime (\tilde {r}_1 )(\tilde {r}_1 - \tilde {r}_2-c ) < p_1^\prime (r_1 )(r_1 - r_2-c )}, p_1 (\tilde {r}_1 ) < p_1 (r_1 )$.

$\therefore {f}'_2 (\tilde {r}_2 \vert \tilde {r}_1 ) - {f}'_2 (r_2 \vert r_1 ) < 0$.

This contradicts ${f}'_2 (\tilde {r}_2 \vert \tilde {r}_1 ) - {f}'_2 (r_2 
\vert r_1 ) = 0$. Therefore, $\tilde {r}_2 > r_2 $ is wrong and $\tilde 
{r}_2 \leq r_2 $ is true. \qed
\end{proof}

\section*{Appendix B: Proof of Theorem \ref{thm13}}
\begin{proof}
It is straightforward that $p_i (t_i )$ monotonically 
increases with $t_i $. Suppose we increase $t_k $ in stage $k$ 
from $t_k^s $ to $t_k^b $ with $t_k^s < t_k^b $ while the thresholds $t_i $ 
in other stages (i.e., $i \ne k$) are fixed. Correspondingly, the rejection rate 
$p_k (t_k )$ grows from $p_k^s $ to $p_k^b $ and the computation cost $f_S $ 
changes from $f_S (t_k^s )$ to $f_S (t_k^b )$. Theorem \ref{thm13} can be proved if 
$f_S (t_k^s ) - f_S (t_k^b ) > 0$.

Define $f_k^L = \sum_{i = 1}^k (r_i+ic) {\left[
{\prod\nolimits_{j = 1}^i {(1 - p_{j - 1} )} } \right]p_i } $, we have
\begin{equation*}
\begin{array}{l@{~}l}
 f_S & = f_{k - 1}^L  + \sum\limits_{i = k}^{S} {(r_i+ic)\left[ {\prod\limits_{j = 1}^i {(1 - p_{j - 1} )} } \right]p_i }\\
& + (T+(S+1)c)\left[ {\prod\limits_{j = 1}^{S+1} {(1 - p_{j-1})} } \right]\\
& = \left[ {\prod\limits_{j = 1}^k {(1 - p_{j - 1} )} } 
\right]\left\{ \sum\limits_{i=k+1}^{S}{(r_i+ic)\left[\prod\limits_{j = k + 1}^{i} (1 - p_{j - 1} )\right] } p_i \right. \\
& \left. + (T+(S+1)c)\left[ {\prod\limits_{j = k+1}^{S+1} {(1 - p_{j-1})} } \right]+r_k p_k\right\}+f_{k - 1}^L.
\end{array}
\end{equation*}



Because $f_{k - 1}^L $ is independent to the threshold parameter in stage 
$k$, so the difference between $f_S (t_k^s )$ and $f_S (t_k^b )$ is also 
independent to $f_{k - 1}^L $. Therefore, 
\begin{equation}
\label{eq69}
\begin{array}{l@{~}l}
& f_S  (t_k^s ) - f_S (t_k^b )\\
&= \left[ {\prod\limits_{j = 1}^k {(1 - p_{j - 1} )} } \right]\left\{ {(r_k+kc)\left( {p_k^s - p_k^b } \right)}\right.\\
& \left. {+ (r_{k+1}+(k+1)c)\left[ {(1 - p_k^s )p_{k + 1}^s - (1 - p_k^b )p_{k + 1}^b } \right] } \right.\\
& + \cdots\\
& + (r_{S}+Sc)\left. {\left[p_S^s {\prod\limits_{j = k + 1}^{S} {(1 - p_{j - 1}^s )}
 - p_S^b \prod\limits_{j = k + 1}^{S} {(1 - p_{j - 1}^b )} } \right]} \right.\\
& \left.+(T+(S+1)c){\left[{\prod\limits_{j = k + 1}^{S+1} {(1 - p_{j - 1}^s )}
 - \prod\limits_{j = k + 1}^{S+1} {(1 - p_{j - 1}^b )} } \right]}\right\}.
\end{array}
\end{equation}

Some items of (\ref{eq69}) can be measured by using the fact that iCascade can 
reject all the true negative sub-windows in training data and the total 
rejection rate $R $ is 1. The rejection rate $R $ consists of the total 
rejection rate $R_{k - 1} = \sum\nolimits_{i = 1}^{k - 1} {\left[ 
{\prod\nolimits_{j = 1}^i {(1 - p_{j - 1} )} } \right]p_i } $ of the first $k 
- 1$ stages and the total rejection rate $\tilde {R}_{k - 1} = 
\sum\nolimits_{i = k}^S {\left[ {\prod\nolimits_{j = 1}^i {(1 - p_{j - 1} 
)} } \right]p_i } + \prod\nolimits_{j = 1}^{S+1} {(1 - p_{j - 1} 
)}$ of stages $k, \ldots ,S$. That is, \\
\begin{array}{l@{~}l}
R &= \sum\limits_{i = 1}^S {\left[ {\prod\limits_{j = 1}^i {(1 - p_{j - 1} 
} )} \right]} p_i+\prod\nolimits_{j = 1}^{S+1} {(1 - p_{j - 1} 
)}\\
&= \sum\limits_{i = 1}^{k-1} {\left[ {\prod\limits_{j = 1}^i {(1 - p_{j - 1} 
} )} \right]} p_i+\sum\limits_{i = k}^{S} {\left[ {\prod\limits_{j = 1}^i {(1 - p_{j - 1} 
} )} \right]} p_i\\
&+\prod\nolimits_{j = 1}^{S+1} {(1 - p_{j - 1} 
)}\\
& = R_{k - 1} + \tilde {R}_{k - 1}  = 1.
\end{array}

Because the total rejection rate $R_{k - 1} $ of the first $k - 1$ stages
does not change with $t_k $, so the total rejection rate $\tilde {R}_{k - 1} 
$ of stages $k, \ldots ,S$ is a constant $\eta $, no matter 
how $t_k $ varies. So if we denote the rejection rates $\tilde {R}_{k - 1} $ corresponding to $p_k^s $ and $p_k^b $ by $\tilde {R}_{k - 1} (t_k^s )$ and $\tilde {R}_{k - 1} (t_k^b )$, respectively, then $\tilde {R}_{k - 1} (t_k^s ) = \tilde 
{R}_{k - 1} (t_k^b ) = \eta $.

Therefore, when $t_k$ grows form $t_{k}^{s}$ to $t_{k}^{b}$, $p_k$ will increase from $p_k^s $ to $p_k^b $, the rejection rate in stage $k$ will increase while the rejection rates in stages $k + 1,...,S$  will decrease or not change (i.e., $\left[ {\prod\nolimits_{j = 1}^i {(1 - p_{j - 1}^s )} } \right]p_i^s \geq \left[ {\prod\nolimits_{j = 1}^i {(1 - p_{j - 1}^b )} } \right]p_i^b$ and $ {\prod\nolimits_{j = 1}^{S+1} {(1 - p_{j - 1}^s )} }  \geq  {\prod\nolimits_{j = 1}^{S+1} {(1 - p_{j - 1}^b )} } $). So we have
\begin{equation}
\label{eq70}
\left[ {\prod\limits_{j = k + 1}^i {(1 - p_{j - 1}^s )} } \right]p_i^s \geq 
\left[ {\prod\limits_{j = k + 1}^i {(1 - p_{j - 1}^b )} } \right]p_i^b,
\end{equation}
where $i =k+1,...,S$.

Based on (\ref{eq70}), $f_S(t_k^s) - f_S(t_k^b) $ in (\ref{eq69}) satisfies:
\begin{array}{l@{~}l}
& f_S(t_k^s) - f_S(t_k^b) \\ 
&> \left[ {\prod\nolimits_{j = 1}^k {(1 - p_{j - 1} )} } \right]\left\{ {r_k\left( {p_k^s - p_k^b } \right)}\right.\\
& \left. {+ r_k\left[ {(1 - p_k^s )p_{k + 1}^s - (1 - p_k^b )p_{k + 1}^b } \right] } \right.\\
& + \cdots\\
& + r_k\left. {\left[p_S^s {\prod\limits_{j = k + 1}^{S} {(1 - p_{j - 1}^s )} 
 - p_S^b \prod\limits_{j = k + 1}^{S} {(1 - p_{j - 1}^b )} } \right]} \right.\\
&  \left.+r_k{\left[{\prod\limits_{j = k + 1}^{S+1} {(1 - p_{j - 1}^s )}
 - \prod\limits_{j = k + 1}^{S+1} {(1 - p_{j - 1}^b )} } \right]}\right\}\\
& = \left[ {\prod\limits_{j = 1}^k {(1 - p_{j - 1} )} } \right]r_k\left\{ 
p_k^s+\sum\limits_{i=k+1}^{S}p_i^s{\prod\limits_{j = k+1}^i {(1 - p_{j - 1}^s )} }\right.\\
&  {+\prod\limits_{j = k + 1}^{S+1} {(1 - p_{j - 1}^s )}- \left[p_k^b+\sum\limits_{i=k+1}^{S} p_i^b \prod\limits_{j = k+1}^i {(1 - p_{j - 1}^b )} \right.} \\
& \left.\left. + \prod\limits_{j = k + 1}^{S+1} {(1 - p_{j - 1}^b )} \right] \right\}\\
& = r_k\left\{{\tilde {R}_{k - 1} (t_k^s ) - \tilde {R}_{k - 1} (t_k^b )} \right\} = r_k\left\{ 
{\eta - \eta } \right\} = 0.
\end{array}

So $f_S(t_k^s) - f_S(t_k^b) > 0$ is proved. \qed
\end{proof}

\ifCLASSOPTIONcaptionsoff
  \newpage
\fi

\end{document}